\documentclass[runningheads]{llncs}

\usepackage[T1]{fontenc}
\def\doi#1{\href{https://doi.org/\detokenize{#1}}{\url{https://doi.org/\detokenize{#1}}}}
\usepackage{graphicx}
\usepackage{listings}
\usepackage{booktabs}       
\usepackage{amsfonts}       
\usepackage{nicefrac}       
\usepackage{microtype}      
\usepackage{xcolor}         

\usepackage[noadjust]{cite}

\usepackage{amssymb}
\usepackage{amsmath}
\usepackage{mathrsfs}
\usepackage[english]{babel}
\usepackage{float}
\usepackage{graphicx}
\usepackage{mathtools}
\usepackage{bm}

\usepackage{filecontents}
\usepackage[font=footnotesize]{subcaption}
\usepackage[font=footnotesize]{caption}
\usepackage{caption}
\usepackage{comment}
\usepackage{multirow}
\usepackage{environ}
\usepackage{setspace}

\usepackage{amsthm, algorithm, algorithmicx}
\usepackage[noend]{algpseudocode}
\usepackage[english]{babel}

\usepackage[inline]{enumitem}

\usepackage{ifthen,version}

\graphicspath{ {img/} }

\usepackage{bbm}

\usepackage{eqparbox}

\newtheorem{innercustomgeneric}{\customgenericname}
\providecommand{\customgenericname}{}
\newcommand{\newcustomtheorem}[2]{%
  \newenvironment{#1}[1]
  {%
   \renewcommand\customgenericname{#2}%
   \renewcommand\theinnercustomgeneric{##1}%
   \innercustomgeneric
  }
  {\endinnercustomgeneric}
}

\newcustomtheorem{customthm}{Theorem}
\newcustomtheorem{customlemma}{Lemma}
\newcustomtheorem{customrmk}{Remark}
\newcustomtheorem{customprop}{Proposition}
\newcustomtheorem{customassu}{Assumption}

\theoremstyle{plain}

\newtheorem{assumption}{Assumption}

\usepackage[symbol]{footmisc}

\algnewcommand\algorithmicinput{\textbf{Input:}}
\algnewcommand\algorithmicoutput{\textbf{Output:}}
\algnewcommand\Input{\item[\algorithmicinput]}%
\algnewcommand\Output{\item[\algorithmicoutput]}%

\DeclareMathOperator*{\argmin}{arg\,min}
\DeclareMathOperator*{\argmax}{arg\,max}

\newcommand{\norm}[1]{\left\lVert#1\right\rVert}

\usepackage{xcolor}

\newcommand{\remove}[1]{}

\newcommand\wl[1]{{\color{blue}{#1}}}

\begin{document}
\title{Sample-efficient Safe Learning for Online Nonlinear Control
with Control Barrier Functions}
\titlerunning{Sample-efficient Safe Learning}
%
\author{Wenhao Luo\inst{1}
\and
Wen Sun\inst{2}
\and
Ashish Kapoor\inst{3}
}
\authorrunning{W. Luo et al.}
%
\institute{Department of Computer Science, University of North Carolina at Charlotte, Charlotte NC 28223, USA \email{wenhao.luo@uncc.edu}\and
Computer Science Department, Cornell University, Ithaca, NY 14853 USA
\email{ws455@cornell.edu}
\and
Microsoft Corporation, Redmond, Washington 98033 USA\\
\email{akapoor@microsoft.com}}
\maketitle              
\begin{abstract}
Reinforcement Learning (RL) and continuous nonlinear control have been successfully deployed in multiple domains of complicated sequential decision-making tasks. However, given the exploration nature of the learning process and the presence of model uncertainty, it is challenging to apply them to safety-critical control tasks due to the lack of safety guarantee. On the other hand, while combining control-theoretical approaches with learning algorithms has shown promise in safe RL applications, the sample efficiency of safe data collection process for control is not well addressed. In this paper, we propose a \emph{provably} sample efficient episodic safe learning framework for online control tasks that leverages safe exploration and exploitation in an unknown, nonlinear dynamical system. In particular, the framework 1) extends control barrier functions (CBFs) in a stochastic setting to achieve provable high-probability safety under uncertainty during model learning and 2) integrates an optimism-based exploration strategy to efficiently guide the safe exploration process with learned dynamics for \emph{near optimal} control performance. We provide formal analysis on the episodic regret bound against the optimal controller and probabilistic safety with theoretical guarantees. Simulation results are provided to demonstrate the effectiveness and efficiency of the proposed algorithm.
\keywords{Robot Safety  \and Safe reinforcement learning \and Safe control.}
\end{abstract}

\section{Introduction}\label{sec:introduction}

The control of safety-critical systems such robotic systems is a difficult challenge under uncertainty and lack of complete information in the real world applications. While Reinforcement Learning (RL) algorithms for long-term reward maximization have achieved significant results in many continuous control tasks \cite{duan2016benchmarking, recht2019tour}, they have not yet been widely applied to safety-critical control tasks as the rigorous safety requirements may be easily violated by intermediate policies during policy learning.
Safe RL approaches \cite{garcia2015comprehensive, achiam2017constrained, amodei2016concrete, moldovan2012safe} with constraints satisfaction have been proposed to encode safety consideration in a modified optimality criterion or in the constrained policy exploration process with external knowledge, e.g. an accurate probabilistic system model \cite{moldovan2012safe, turchetta2016safe}. 

On the other hand, model-based approaches utilizing Model Predictive Control (MPC) or Lyapunov-based methods have seen a number of success in demanding control tasks 
under different constraints with accurate system models \cite{hogan2020reactive, zhu2019chance, zeng2020safety} or approximated dynamics \cite{williams2017information}. Meanwhile, enforcing safety in terms of \emph{set forward invariance} has become an active research area using Lyapunov functions \cite{berkenkamp2017safe} and control barrier functions (CBFs) with perfectly known system models \cite{ames2017control, zeng2020safety, ames2019control} or noisy models with known distributions \cite{luo2020multi, taylor2020adaptive, gurriet2018towards, clark2019control, lyu2021probabilistic}. However, these control-based approaches still require known system model uncertainty and could be overly conservative for system behaviors in presence of large uncertainty.

For this purpose, integrating data-driven learning-based approaches with model-based safe control has received significant attention to achieve model uncertainty reduction while ensuring provable safety \cite{berkenkamp2017safe, koller2018learning, wang2018safe, fisac2018general, khojasteh2020probabilistic, cheng2019end, taylor2020learning, choi2020reinforcement, ohnishi2019barrier, cheng2020safe}. 
For example, standard CBFs \cite{wang2018safe} and robust CBFs \cite{cheng2020safe} have been employed  in conjunction with Gaussian Processes to ensure robot safety with uncertainty learned online. 
The process often involves safe policy exploration with data collection from a nominal dynamics model and iteratively reduce learned model uncertainty over time to expand certified safety region of the system's state space \cite{berkenkamp2017safe, wang2018safe, khojasteh2020probabilistic, choi2020reinforcement, taylor2020learning}. 
However, such exhaustive data collection for safe learning could suffer from poor scalability and low efficiency for primary task. For example, instead of densely sampling over the space, it may be more beneficial to guide the safe exploration process towards task-prescribed policy optimization.
Recent work \cite{cheng2019end} incorporates safe learning using Gaussian Process (GP) and CBF into a model-free RL framework (RL-CBF), so that the guided exploration process will not only learn model uncertainty impacting safe behaviors but also optimizing the policy performance. 
Similar idea was proposed in \cite{marvi2021safe} to achieve proactive safe planning using CBF and optimal performance with an off-policy RL algorithm.

To the best of our knowledge, there is no work on regret guarantee reflecting sample efficiency of \emph{safe learning} for nonlinear control with unknown model uncertainty.
Note that safe learning for controlling unknown nonlinear dynamical systems is indeed a challenging problem, since safety and exploration are contrast to each other. 
In this paper, we propose a \emph{provably correct} method that handles both sample efficient safe learning and online nonlinear control task in partially unknown system dynamics. In particular, we develop an Optimism-based Safe Learning for Control framework that integrates 1) stochastic discrete-time control barrier functions (CBF) to ensure 
safety with high probability under uncertainty, and 2) an optimism-based exploration strategy that enjoys a formally provable regret bound. 
To leverage between safe exploration and exploitation, the framework utilizes an optimism-based exploration strategy in face of uncertainty \cite{kakade2020information}. This encourages efficient dynamics exploration and simultaneously synthesizes with model-based nonlinear control algorithms to safely optimize the policy performance under the learned dynamics. 
Compared to existing works on safe learning or safe RL \cite{wang2018safe, berkenkamp2017safe, taylor2020learning, cheng2019end, choi2020reinforcement}, our framework is able to simultaneously guarantee safety with high probability (remain in safe set during learning and execution), performance (task rewards maximization), and sample efficiency (near-optimal regret bound). The effectiveness of the algorithm is validated through simulation results on a unicycle mobile robot and an inverted pendulum.
Our \textbf{main contributions} are: 1) a provably sample efficient episodic online learning framework that integrates safe model-based nonlinear control approaches with optimism-based exploration strategy to simultaneously achieve safe learning and policy optimization for online nonlinear control tasks, and 2) rigorous theoretical analysis of guaranteed probabilistic safety under learned uncertainty and near-optimal online learning and policy performance with proved regret bound.
\vspace{-0.3cm}

\section{Preliminaries}\label{sec:problem}

\subsection{Dynamical System and Stochastic Control}
Consider the following partially \emph{unknown} discrete-time control-affine system dynamics with state $x\in \mathcal{X}\subset \mathbb{R}^n$ and control input $u\in\mathcal{U}\subset\mathbb{R}^m$ for a discrete time index $h\in\mathbb{N}$
\begin{equation}\label{eq:sys_unknown}
    x_{h+1} = \hat{f}(x_h, u_h)  + d(x_h, u_h) + \varepsilon_h, \quad \varepsilon_h \sim \mathcal{N}(0, \Sigma_{\sigma})
\end{equation}
where $\hat{f}: \mathcal{X}\times\mathcal{U}\mapsto \mathbb{R}^n$ is the known nominal discrete dynamics affine in the control input as $\hat{f}(x_h, u_h)=\hat{F}(x_h)+\hat{G}(x_h)u_h$. We assume $\hat{F}: \mathbb{R}^n\mapsto\mathbb{R}^n,\;\hat{G}: \mathbb{R}^n\mapsto\mathbb{R}^{n\times m}$ are locally Lipschitz continuous and the relative degrees of the nominal model and the actual system are the same, which are common assumptions as in \cite{taylor2020learning, choi2020reinforcement}.
The system dynamics in this form is general and could describe a large family of nonlinear systems, e.g. 3-dof differential drive vehicles with unicycle dynamics \cite{pickem2017robotarium, wang2017safety}, 12-dof quadrotors with underactuated system \cite{wang2018safe}, bipedal robots, automotive vehicle, and Segway robots \cite{ames2019control, taylor2020learning}.

To describe a prior unknown model error representing uncertainty,
$d: \mathcal{X}\times\mathcal{U}\mapsto \mathbb{R}^n$ denotes the unmodelled part of the system dynamics 
which is unknown.
This unmodelled part $d(x,u)$ could represent state (action)-dependent external motion disturbances \cite{wang2018safe} or system model error due to parameter mismatch \cite{cheng2019end, taylor2020learning}.
$\varepsilon_h$ is i.i.d noise sampled from a known Multivariate Gaussian distribution with the covariance matrix $\Sigma_\sigma = \text{diag}(\sigma_1^2,\ldots,\sigma_n^2)$, i.e., $\sigma_1,\ldots,\sigma_n$ are known to the learner. For notation simplicity, we denote the stochastic state transition as $P(\cdot | x, u)$.
In particular, 
we assume that $d(x,u)$ is modelled by the nonlinear model $d(x,u) := W^\star \phi(x,u)$ where $\phi: \mathcal{X}\times\mathcal{U}\mapsto \mathbb{R}^r$ is a known nonlinear feature mapping, e.g. Random Fourier Features (RFF) \cite{rahimi2007random}, 
and 
the linear mapping $W^\star\in \mathbb{R}^{n\times r}$ is the unknown system parameters that need to be learned. Such model has been studied in \cite{kakade2020information} for unconstrained online control and in \cite{mania2020active} for pure system identification. The model is flexible enough to describe the classic linear dynamical systems, nonlinear dynamical systems such as high order polynomials, and piece-wise linear systems \cite{mania2020active}.
The control task is described by a cost function. Given an immediate cost function $c: \mathcal{X} \times \mathcal{U}\mapsto \mathbb{R}^+$, the primary objective is 
{\footnotesize
\begin{align}
 \min_{\pi\in\Pi}J^{\pi }(x_0;c, W^\star) = \min_{\pi\in\Pi} \mathbb{E}\bigg[\sum_{h=0}^{H-1}c(x_h, u_h)|\pi, x_0, W^\star\bigg] \label{eq:obj:primary}
\end{align}
}
where $\mathbf{x}_0\in\mathcal{X}$ is a given starting state and $\Pi$ is some set of pre-defined feasible controllers. Each controller (or a policy) is a mapping $\pi\in \Pi:\mathcal{X}\mapsto \mathcal{U}$. We denote $J^{\pi}(x; c, W)$ as the expected total cost of a policy $\pi$ under cost function $c$, initial state $x$, and the dynamical system in Eq.~\ref{eq:sys_unknown} whose $d(x,u)$ is parameterized by $W$. 
To achieve optimal task performance, we need to learn the unmodelled $d(x,u)$ by taking samples to approximate to the true linear mapping $W^\star$. On the other hand,
simply optimizing the cost function is not enough for safety critical application. 
Below we consider a specific formulation that uses Control Barrier Functions (CBFs) \cite{ames2019control} to enforce safety constraints.

\subsection{Discrete-time Control Barrier Functions For Gaussian Dynamical Systems}
We introduce CBFs in this section that are specialized to \emph{discrete-time Gaussian stochastic dynamics} and provide 
safety analysis in a stochastic setting. Note that most existing continuous or discrete-time CBFs are defined only for deterministic system, e.g. \cite{ames2019control, wang2018safe, taylor2020learning, choi2020reinforcement, agrawal2017discrete, cheng2019end}. For stochastic system, we introduce a new CBF definition and 
prove the safety guarantee
\emph{with high probability} where the high probability is corresponding to the system Gaussian noise (Proposition~\ref{prop:forward_invariant}). 
Consider a stochastic Gaussian discrete-time dynamics described in Eq.~\ref{eq:sys_unknown}. 
A desired safety set $x\in\mathcal{S}\subset \mathcal{X}$ can be denoted by the following safety function $h^s:\mathbb{R}^n\mapsto\mathbb{R}$ 
\begin{equation}\label{eq:safeset_general}
\mathcal{S} =\{x \in \mathbb{R}^n : h^s(x)\geq 0\}
\end{equation} 
Formally, a safety condition is forward invariant if $x_{h=0}\in \mathcal{S}$ implies $x_h\in \mathcal{S}$ for all time step $h>0$ with some designed controller $u\in\mathcal{U}$. Control barrier functions (CBF) \cite{ames2019control} are often used to derive such designed controllers that enforce the forward invariance of a set of the system state space.
Due to the stochasticity of the dynamics, we need to take the unbounded noise into consideration and thus we define  
a stochastic discrete-time CBF $h^s(\cdot)$ with the following condition. 
\begin{definition}\label{def:cbf}[Discrete-time Control Barrier Function under Known Gaussian Dynamics] Assume $h^s(\cdot)$ is $L$-Lipschitz continuous when $x\in\mathcal{X}$ is bounded. Given $\delta_s\in (0,1)$ and time horizon $H$, let $\mathcal{S}$ be the $0$-superlevel set of $h^s:\mathbb{R}^{n}\to\mathbb{R}$ which is a continuously differentiable function. We call $h^s(\cdot)$ a stochastic discrete-time control barrier function (CBF) for dynamical system Eq.~\ref{eq:sys_unknown} if there exists a $\eta \in (0,1)$, such that for all time steps $h=0,\ldots,H-1$,  given any $x\in\mathcal{S}$:
{\footnotesize
\begin{equation}\footnotesize
\sup_{u\in\mathcal{U}} \left[ h^s\left( \hat{f}(x,u) + d(x,u) \right) - L\bar{\sigma} \sqrt{ 2n \ln\left( \frac{H n}{\delta_s}\right)} \right.
 - h^s(x ) \Bigg] \geq  - \eta h^s(x)
\label{eq:constraint_barrier}
\end{equation} 
}
\end{definition}
where $\bar{\sigma}=\max\{\sigma_1,\ldots,\sigma_n\}$. Note that different from conventional CBF \cite{ames2019control, cheng2019end} that is defined with respect to deterministic transition, the above definition takes the stochasticity into consideration.  Also note that when $\bar{\sigma}\to 0$, i.e., the Gaussian dynamical system in Eq.~\ref{eq:sys_unknown} becomes a near deterministic system, then the above definition converges to the usual discrete-time CBF definition \cite{cheng2019end, zeng2020safety}. 
We now show that under Definition~\ref{def:cbf}, the resulting system state will stay in the safe set $\mathcal{S}$ 
with probability at least $1-\delta_s$.
\begin{proposition} [Guaranteed Safety
with High Probability] \label{prop:forward_invariant}
Consider a CBF $h^s(\cdot)$ in Definition~\ref{def:cbf}. Given $x_0 \in \mathcal{S}$, consider any policy $\pi:\mathcal{X}\to\mathcal{U}$ such that at any state $x$, this policy outputs an action $u = \pi(x)$ that satisfies the constraint Eq.~\ref{eq:constraint_barrier}.
Then executing $\pi$ to generate a trajectory starting at $x_0$: $\tau = \{x_0,u_0,\dots, x_{H-1}, u_{H-1}\}$, with probability at least $1-\delta_s$ we have $ h^s(x_h) \geq 0$ for all $h\in [H]$, i.e., all states on the trajectory belong to the safe set $\mathcal{S}$.
\end{proposition}\vspace{-0.3cm}
\begin{proof}
For notation simplicity, let us denote $f(x,u) := \hat{f}(x,u) + d(x,u)$. 
First consider as follows the probability of the event: $\exists h=1,\ldots,H, i=1,\ldots,n$ such that $\epsilon_h[i] \geq p$. Via union bound and the normal distribution's property, we have:
\begin{align*}
\mathbb{P}\left( \exists h\in [H], i\in [n], \text{ s.t., }\epsilon_h[i] \geq p \right) \leq H n \exp(- p^2 /(2\bar{\sigma}^2))
\end{align*} Let us set the failure probability to be $\delta_s$, i.e., $Hn \exp(-p^2 /(2\bar{\sigma}^2)) = \delta_s$. Solve for $p$ and we indeed have $p := \bar{\sigma}\sqrt{ 2\ln\left( \frac{H n}{\delta_s} \right) }$. Below we conditioned the event that for all $h\in [H]$  and $i\in [n]$, we have $|\epsilon_h[i]| \leq p$. Note that this event happens with probability at least $1-\delta_s$.

Due to the Lipschitz assumption with a Lipshcitz constant $L$ and the assumption that $h^s(\cdot)$ is differentiable, we have that for the system $x_{h+1}=f(x_h, u_h)  + \epsilon_h$ at any time $h\in [H]$, 
{\footnotesize
\begin{equation*}
\begin{split}
h^s( x_{h+1} ) = h^s( f(x_h, u_h ) ) + \nabla h^s(\xi)^{\top} \epsilon_h  
\geq h^s( f(x_h, u_h)  )  - L  \bar{\sigma} \sqrt{2 n \ln\left( \frac{H n}{\delta_s}\right) }
\end{split}
\end{equation*}
}
Since the control policy satisfies the control barrier function constraint in Eq.~\ref{eq:constraint_barrier}, we must have:
{\footnotesize
\begin{equation*}
\begin{split}
h^s( x_{h+1} ) - h^s(x_h) \geq \left[ h^s( f(x_h,u_h))  - L \bar{\sigma} \sqrt{2 n \ln\left( \frac{H n}{\delta_s}\right) }\right.  
- h^s(x_h)\Bigg] \geq  -\eta h^s(x_h)
\end{split}
\end{equation*}  
}
Namely, we have:
\begin{align*}
h^s( x_{h+1} ) \geq (1-\eta) h^s(x_h) \geq (1-\eta)^{h+1} h^s(x_0) \geq 0,
\end{align*} under the condition that $h^s(x_0) \geq 0$. The above argument holds for all $h\in[H]$ which thus concludes the proof.
\end{proof}
In presence of uncertainty on the considered system dynamics, it is noted that an initially chosen control barrier function $h^s(\cdot)$ in Eq.~\ref{eq:constraint_barrier} describing the desired safety set $\mathcal{S}$ should exist and remain valid during the safe learning task. Thus similar to \cite{taylor2020learning}, we make the following assumption without loss of generality.
\begin{assumption}
If a function $h^s(\cdot)$ is a valid stochastic CBF as defined in Eq.~\ref{eq:constraint_barrier} for the nominal stochastic system dynamics $\hat{f}(x_h, u_h)+\epsilon_h$, then it is a valid stochastic CBF for the uncertain stochastic dynamical system in Eq.~\ref{eq:sys_unknown}. Mathematically this implies that
{\footnotesize
\begin{multline*}
    \; \qquad \sup_{u\in\mathcal{U}} \left[ h^s\left( \hat{f}(x,u) \right) - L\bar{\sigma} \sqrt{ 2n \ln\left( \frac{H n}{\delta_s}\right)} \right.
 - h^s(x ) \Bigg] \geq  - \eta h^s(x) \\
 \implies \sup_{u\in\mathcal{U}} \left[ h^s\left( \hat{f}(x,u) + d(x,u) \right) - L\bar{\sigma} \sqrt{ 2n \ln\left( \frac{H n}{\delta_s}\right)} \right.
 - h^s(x ) \Bigg] \geq  - \eta h^s(x)
\end{multline*}\label{ass:valid_cbf}
}
\vspace{-0.4cm}
\end{assumption}
\noindent
Similar to the deterministic CBF \cite{taylor2020learning}, our assumption indicates that a set within the state space that can be kept safe for the nominal model of the stochastic system can also be kept safe for the actual uncertain system. This typically amounts to the assumption that the relative degree of $h^s(\cdot)$ for the nominal model is the same as the relative degree of $h^s(\cdot)$ for the actual system, so that the chosen CBF $h^s(\cdot)$ remains a valid CBF during learning \cite{taylor2020learning}. Note that in order to reduce conservativeness during learning for improved sample efficiency, a strategic safe exploration process is still needed to efficiently learn the unmodelled $d(\cdot)$ while ensuring safety as the predicted model of $d(\cdot)$ being updated overtime.
\vspace{-0.5cm}

\subsection{Learning Objective}

If assuming known dynamics $d(\cdot)$ in Eq.~\ref{eq:sys_unknown}, then the safe nonlinear control problem can be modeled as follows:
{\footnotesize
\begin{align}\footnotesize
\label{eq:ground_truth_opt}
&\min_{\pi\in\Pi} J^{\pi}(x_0; c) \\
 s.t. \; & \;h^s\left(\hat{f}(x, \pi(x)) + d(x,\pi(x))  \right)  
 - L\bar{\sigma} \sqrt{2 n \ln\left( \frac{H n}{\delta_s}\right)}  - h^s(x ) \geq  - \eta h^s(x) \label{eq:ground_truth_opt_const}
\end{align}
}
Assume the above constrained optimization problem is feasible, and denote $\pi^\star$ as the optimal policy. Then, via Proposition~\ref{prop:forward_invariant}, we know that with probability at least $1-\delta_s$, $\pi^\star$ generates a trajectory whose states are all in the safe set $\mathcal{S}$.  
However, as the unmodelled part $d(x,u)$ is initially unknown, we cannot directly solve the above constrained optimization program using standard RL or MPC approaches. Instead we need to learn $d$ (more specifically, the unknown linear mapping $W^\star\in \mathbb{R}^{n\times r}$) online using data-driven approach for probabilistic safety guarantee as well as cost minimization. 

Following the standard online learning setup \cite{hazan2016introduction, wagener2019online},
in our episodic finite horizon learning framework we start with some initialization $\overline{W}_0$ which is used to parameterize $d_0(x,u) := \overline{W}_0 \phi(x,u)$
. At every episode $t$, the learner will propose a policy $\pi_t\in \Pi$ (probably based on the current guess $d_t(x,u)$ with $\overline{W}_{t}$), execute $\pi_t$ in the real system to generate one trajectory $\{x_0^{t}, u_0^{t},\dots, x_{H-1}^t, u_{H-1}^t\}$ for $H$ time steps; the learner then incrementally updates model parameter to $\overline{W}_{t+1}$ using observations from all of the past trajectories, and move to the next episode $t+1$ starting from the same initial state $x\leftarrow x_0$. 
The goal is to ensure that $\pi_t$ is safe (with high probability) in terms of satisfying CBF constraint Eq.~\ref{eq:constraint_barrier}, and also optimize the cost function over episodes:
\begin{align}\footnotesize
\text{Regret}_{T} := \sum_{t=0}^{T-1} \sum_{h=0}^{H-1} c(x_h^t, u_h^t) -  \sum_{t=0}^{T-1} J^{\pi^\star}(x_0; c) = o(T)
\label{eq:regret_def}
\end{align} 
Namely, comparing to the best policy $\pi^\star$ (i.e., the optimal solution of the constrained optimization program in Eq.~\ref{eq:ground_truth_opt} if assuming perfect model information), the cumulative regret grows sublinearly with respect to the number of episodes $T$. This goal implies that when $T\to\infty$ in a long run, the average episodic cost incurred by the learner is the same as the episodic cost incurred by the best policy $\pi^\star$ generated with the ground-truth dynamics Eq.~\ref{eq:sys_unknown}. \vspace{-0.3cm}

\section{Algorithm and Analysis}\label{sec:algo}


\subsection{Calibrated Model and Approximate Safety Guarantee}\label{sec:initialization}

Due to the unknown $d$, our initial guess $d_0$ using $\overline{W}_0$ estimated from pre-collected data could be arbitrarily bad, and hence it is impossible to achieve safe learning while simultaneously optimizing the regret without further assumptions. 
To ensure high-confidence safety (which could be overly conservative before exploration) starting with initially learned uncertainty from a pre-collected training data set, existing literature on safe learning often made the standard assumption that the prior calibrated model of $d$ could yield initial safe policy to start the exploration process (e.g. \cite{cheng2019end, berkenkamp2017safe, taylor2020learning, wang2018safe, khojasteh2020probabilistic, choi2020reinforcement}), for example, by assuming the true $d$ stays within the high confidence region estimated from Gaussian Processes using an initial data set.
Hence without loss of generality, we make the typical assumption similar to existing literature of a calibrated initial model in the following.

Given $N$ triples $(x_i, u_i, x_i')_{i=1}^N$ with $x'\sim P(\cdot | x,u)$ as a set of pre-collected initial training data,
we can compute $\overline{W}_0$--- the initialization  parameters of $d(x,u)$ via ridge linear regression under known feature mapping $\phi: \mathcal{X}\times\mathcal{U}\mapsto \mathbb{R}^r$: 
{\footnotesize
\begin{align}\label{eq:w0}\footnotesize
\overline{W}_0 = \argmin_{W} \sum_{i=1}^N \left\| W\phi(x_i,u_i) - (x_i' - \hat{f}(x_i,u_i))  \right\|_2^2  + \lambda \|W\|_F^2
\end{align} 
}
where $\lambda$ is a regularizer parameter and $\|W\|_F$ is the Frobenius norm of the model parameter matrix $W\in\mathbb{R}^{n\times r}$.
Denote the initial empirical regularized covariance matrix as 
\begin{equation}\label{eq:v0}\footnotesize
V_0 = \sum_{i=1}^N \phi(x_i,u_i)\phi(x_i,u_i)^{\top} + \lambda I
\end{equation}

\begin{assumption}
\label{lem:valid_initialization} (Calibrated model) With $\overline{W}_0, V_0$ from the initial data $(x_i, u_i, x_i')_{i=1}^N$ and $\epsilon,\delta\in(0,1)$ where $\epsilon$ is a small error term, we can build the initial confidence ball describing the uncertain region of the linear mapping $W^\star$ with probability at least $1-\delta$ as follows:
\begin{align}\label{eq:ball_0}\footnotesize
\texttt{Ball}_0 = \left\{ W: \left\| (W - \overline{W}_0) V_0^{1/2}  \right\|_2 \leq \beta, \quad \| W  \|_2 \leq \|W^\star\|_2     \right\}
\end{align} where $\beta$
is a hyper-parameter describing an appropriate confidence radius.
Then for all  $\widetilde{W} \in \texttt{Ball}_0$:
{\footnotesize
\begin{align*}\footnotesize
\forall x,u\in\mathcal{X}\times\mathcal{U}: \; \left\| \left(\widetilde{W} - W^\star\right) \phi(x,u) \right\|_2 \leq \mathcal{O}\left( \epsilon \right).
\end{align*}
}
\end{assumption}
See \wl{Appendix Section~\ref{sec:app:lem:valid_initialization}} for details. This entails for any $\widetilde{W} \in \texttt{Ball}_0$, its prediction $\widetilde{d}(x,u)=\widetilde{W}\phi(x,u)$ for any $x,u$ is reasonably good to the true prediction ${d}(x,u)=W^\star\phi(x,u)$ from $W^\star$ (note that however we cannot guarantee $\overline{W}_0$ will be close to $W^\star$ in terms of $\ell_2$ norm).
As discussed in the following Theorem~\ref{prop:approximate_safe}, 
this assumption ensures when we control the dynamical system using CBF with any model $\widetilde{W} \in \texttt{Ball}_0$, we can ensure safety update to $\mathcal{O}(\epsilon)$ since $W^\star\in \texttt{Ball}_0$ with high probability. 
As more data are collected by each future training episode $t$, we show in Section~\ref{subsec:algorithm_regret_analysis} the decreased confidence region $\texttt{Ball}_t$ covers true $W^\star$ with high probability, yielding less conservative safe learning behavior over time and achieving provable sample efficiency during learning.

\begin{theorem}[Policy for Approximate High-Probability Safety Guarantee with Learned Dynamics] \label{prop:approximate_safe}
Under Assumption~\ref{lem:valid_initialization}, 
consider any $\widetilde{W}\in\texttt{Ball}_0$, and define any policy $\pi_{s}:\mathcal{X}\mapsto\mathcal{U}$ that satisfies the CBF constraint parameterized by $\widetilde{W}$, i.e., 
{\footnotesize
\begin{multline}\label{eq:cbf_with_approximate_model}
\forall x \in \mathcal{X}: \pi_s(x) \in \mathcal{U}_{x} := \bigg\{u:  h^s\left( \hat{f}(x,u) + \widetilde{W}\phi(x,u) \right) -\\ 
L{\bar{\sigma}} \sqrt{ 2n \ln\left( \frac{H n}{\delta_s}\right)}  \geq  (1- \eta)    h^s(x) \bigg\}
\end{multline}
}
Then with probability at least $1-\delta_s$, starting at any safe initial state $h^s(x_0)\geq 0$, $\pi_s$ generates a safe trajectory $\{x_0,u_0,\dots, x_{H-1}, u_{H-1}\}$, such that for all time steps $h\in [H]$, $h^s(x_h) \geq -\mathcal{O}( \frac{L \epsilon}{\eta} )$, where $L$ is the Lipschitz constant of $h^s(\cdot)$ under bounded $x\in\mathcal{X}$.
\end{theorem}
The detailed proof (See \wl{Appendix Section~\ref{sec:app:prop:approximate_safe}}) follows from Assumption~\ref{lem:valid_initialization} and Proposition~\ref{prop:forward_invariant}. 
Next, we show 
how to search for a policy using an optimism-based algorithm that performs as good as the best benchmark $\pi^\star$ in the sense of minimizing regret defined in Eq.~\ref{eq:regret_def} and subject to Eq.\ref{eq:cbf_with_approximate_model}.

\subsection{Optimism-based Safe Learning for Control }\label{subsec:algorithm_regret_analysis}

To achieve no-regret performance for efficient safe learning for control, we leverage the LC$^3$ algorithm \cite{kakade2020information} for strategic exploration.
However, LC$^3$ is designed for unconstrained optimization and hence not suitable for safety-critical applications.
Here we adopt our CBF constraint Eq.~\ref{eq:cbf_with_approximate_model} to select safe exploration policy.
Similar to LC$^3$, we need to leverage the principle of optimism in the face of uncertainty to achieve small regret. We propose the following framework of Optimism-based Safe Learning for Control (Algorithm~\ref{alg:optmism_based_learning}).
\begin{algorithm}[h]\footnotesize
\caption{Optimism-based Safe Learning for Control}
\label{alg:optmism_based_learning}
\begin{spacing}{1.2}
    \begin{algorithmic}[1]
    \Input{CBF $h^s$, cost function $c$, initial data $(x_i, u_i, x_i')_{i=1}^N$, initial confidence region $\texttt{Ball}_0$ with $\overline{W}_0, \Sigma_0$, number of training episodes $T$, horizon $H$, regularizer $\lambda$, initial state $x_0$}
    \Output{a sequence of policies for $t= 0,...,T$}
    \For {$t=0,\ldots,T$}{}
    \State $x^t_0\leftarrow x_0$
    \State Sample $\widetilde{W}_t\sim \mathcal{N}(\overline{W}_t, \Sigma_t^{-1})$ in $\texttt{Ball}_t$\hfill $\#$ \texttt{Thompson Sampling for Exploration}\label{ln:thompson_sampling}
    \State $\pi_s^t \leftarrow \argmin_{\pi\in \Pi_{\widetilde{W}}} J^{\pi}(x^t_0;c,\widetilde{W}_t) \text{ with $\Pi_{\widetilde{W}}$ specified in Eq.~\ref{eq:safe_policy_class}} $\hfill $\#$ \texttt{Safe MPC Planning}\label{ln:safe_mpc}
    \State Execute $\pi_s^t$ to sample a trajectory $\tau^t:=\{x_h^t,u_h^t,c_h^t,x_{h+1}^t\}_{h=0}^{H-1}$ \hfill $\#$\label{ln:execute}
    \texttt{Execution and Data Collection}
    \State $\overline{W}_{t+1},\Sigma_{t+1}\leftarrow\text{Update }\texttt{Ball}_{t+1}$ \hfill $\#$ \texttt{Model Update}\label{ln:ball_update}
    \EndFor Return a sequence of policies for $t= 0,...,T$
    \end{algorithmic}
    \end{spacing}
\end{algorithm}

At the beginning of each episode $t$, given all previous trajectories, $\tau^{i} = \{ x_0^i, u_0^i, \dots, x_{H-1}^i, u_{H-1}^i, x_H^i \}$ from episode $i = 0$ to $t-1$, we perform ridge linear regression to find $\overline{W}_t$ in Line~\ref{ln:ball_update}, i.e.,
\vspace{-0.3cm}
{\footnotesize
\begin{multline}\label{eq:alg_barWt}
    \overline{W}_t  =\argmin_{W} \sum_{i=1}^N \left( W\phi(x_i,u_i) - (x'_{i} - \hat{f}(x_i,u_i))  \right)^2 \\
    + \sum_{i=0}^{t-1} \sum_{h=0}^{H-1} \left( W\phi(x_h^i,u_h^i) - (x^i_{h+1} - \hat{f}(x_h^i,u_h^i))  \right)^2  + \lambda \|W\|_F^2
\end{multline}
}
and we have the shape of the estimate region as 
\begin{equation}\footnotesize
\Sigma_t = V_0 + \sum_{i=0}^{t-1} \sum_{h=0}^{H-1} \phi(x_h^i,u_h^i)\phi(x_h^i, u_h^i)^{\top}
\end{equation}
where we use the data from trajectories $\{\tau^{i}\}_{i=0}^{t-1}$ and the initial data $(x_i, u_i, x_i')_{i=1}^N$ for computing $\texttt{Ball}_0$. Then the confidence region of $W$ is defined as:
\begin{align}\footnotesize\label{eq:ball_t}
\texttt{Ball}_t = \texttt{Ball}_0 \cap \left\{ W: \| (W - \overline{W}_t) \Sigma_t^{1/2} \|_2 \leq \beta_t \right\}
\end{align} 
with the confidence radius parameter $\beta_t$ defined as:
\begin{equation}
\begin{split}
\beta_t : = \sqrt{\lambda}  C_1 + \bar{\sigma} \sqrt{ 8 n \ln(5) + 8 r \ln\left( 1 + (TH+N)/\lambda  \right)  + 8 \ln(1/\delta)}
\end{split}
\end{equation}
where $C_1$ is a design hyper-parameter (see \wl{Appendix Section~\ref{sec:app:lem:valid_initialization}}).
The uncertainty region $\texttt{Ball}_t$ decreases as more data are collected over episodes. 
LC$^3$ \cite{kakade2020information} shows that with probability $1-\delta$, for all $t$, $W^\star\in \{W: \|(W - \overline{W}_t) \Sigma_t^{1/2} \|_2 \leq \beta_t\}$. Then under our Assumption~\ref{lem:valid_initialization}, we have the following proposition proving the existence of such intersection with true $W^\star$ living in the confidence interval $\texttt{Ball}_t$ (Eq.~\ref{eq:ball_t}).
\begin{proposition}\label{prop:intersect_confidence_interval}
Given the uncertainty regions $W^\star\in \{W: \|(W - \overline{W}_t) \Sigma_t^{1/2} \|_2 \leq \beta_t\}$ (Proof of Lemma B.5 in \cite{kakade2020information}) and $\texttt{Ball}_0$ (Eq.~\ref{eq:ball_0}) with the probability of $\text{Pr}(W^\star\in \{W: \|(W - \overline{W}_t) \Sigma_t^{1/2} \|_2 \leq \beta_t\})\geq 1-\delta$ and $\text{Pr}(W^\star\in \texttt{Ball}_0)\geq 1-\delta$, then for all episode $t$ we have
{\footnotesize
\begin{equation}
\begin{split}
    \text{Pr}\left( W^\star \in \texttt{Ball}_t:=  \texttt{Ball}_0 \cap \left\{ W: \| (W - \overline{W}_t) \Sigma_t^{1/2} \|_2 \leq \beta_t\right\} \right) 
    \geq 1-2\delta
\end{split}
\end{equation}
}
\end{proposition}
where $\text{Pr}(\cdot)$ denotes the probability of an event. 
Detailed proof is provided in \wl{Appendix Section~\ref{sec:app:prop:intersect_confidence_interval}}.
This ensures as more data are collected,
given any model $\widetilde{W} \in \texttt{Ball}_t\Rightarrow\widetilde{W} \in \texttt{Ball}_0$, we can constrain our policy class $\Pi$ based on CBF constraint under $\widetilde{W}$ (Eq.~\ref{eq:cbf_with_approximate_model}) in Theorem~\ref{prop:approximate_safe}, i.e., we denote $\Pi_{\widetilde{W}}$ as follows:
{\footnotesize
\begin{equation}\label{eq:safe_policy_class}
\begin{split}
\Pi_{\widetilde{W}} = \bigg\{ \pi_s\in \Pi : \forall x\in\mathcal{X},\pi_s(x) \in \Big\{u:  h^s\left( \hat{f}(x,u) + \widetilde{W}\phi(x,u) \right) \\
- L\bar{\sigma} \sqrt{ 2n \ln\left( \frac{H n}{\delta_s}\right)} \geq  (1- \eta) h^s(x) \Big\}\bigg\}
\end{split}
\end{equation}
}
As discussed in \cite{agrawal2017discrete, zeng2020safety, cheng2019end}, a nonlinear discrete-time barrier function $h^s(\cdot)$ generally makes the constrained optimization process (Line~\ref{ln:safe_mpc} in our case subject to Eq.~\ref{eq:safe_policy_class}) a nonlinear programming problem (NLP) that is possibly non-convex, while with an affine barrier function $h^s(\cdot)$ it could become tractable as a convex optimization problem \cite{cheng2019end}.
Such an optimization problem could be solved by using standard MPC approach. 
To obtain an approximate solution for improved computation efficiency, with $\hat{f}$ affine in control and assumed local Lipshitz continuity of the dynamics and feature mapping function $\phi$, Eq.~\ref{eq:safe_policy_class} could also be approximated by linear control constraint w.r.t. $u$ as follows to define the admissible control space
\vspace{-0.3cm}
{\footnotesize
\begin{multline}\footnotesize\label{eq:cbf_learned_approximate_ctrl_constraints}
L_{\hat{F}}^{\Delta}h^s(x)+L_{\hat{G}}^{\Delta}h^s(x)u 
-  L\bar{\sigma} \sqrt{ 2n \ln\left( \frac{H n}{\delta_s}\right)}\\
\geq - \eta h^s(x) 
+|\Delta h^s(x)\widetilde{W}\phi(x,u^\star)|+|\Delta h^s(x)\widetilde{W}L_{x,\phi}(u^+-u^-)| 
\end{multline}
}
where $L_{\hat{F}}^{\Delta}h^s(x)$ and $L_{\hat{G}}^{\Delta}h^s(x)$ are discrete-time Lie-derivatives of $h^s(x)$ obtained through Taylor's theorem along $\hat{F}(x)$ and $\hat{G}(x)$ respectively. 
$L_{x,\phi}$ is the local Lipschitz constant vector for the known feature mapping $\phi$ w.r.t. $u$ at $x$.
$\Delta h^s(x)$ is the discrete derivative of $h^s$ and $u^\star, u^+,u^-$ are the nominal, max and min value of $u$ respectively. \wl{Appendix Section~\ref{sec:app:eq:safe_policy_class}} proves that any $u$ satisfying Eq.~\ref{eq:cbf_learned_approximate_ctrl_constraints} ensures $u\in \pi_s(x)$ in Eq.~\ref{eq:safe_policy_class}, thus constructing the safe policy class $\Pi_{\widetilde{W}}$.
Note that if the ground truth dynamics $d$ is only state-dependent as assumed in \cite{wang2018safe, berkenkamp2017safe, cheng2019end}, then Eq.~\ref{eq:cbf_learned_approximate_ctrl_constraints} is also linear in control where $L_{x,\phi}=\mathbf{0}$ with feature mapping $\phi(x)$.
Now we select model and safe policy optimistically at each episode $t$: 
{\footnotesize
\begin{align}\label{eq:double_min}
\left( W_t, \pi^t \right) := \argmin_{ \widetilde{W} \in \texttt{Ball}_t }\argmin_{\pi\in\Pi_{\widetilde{W}}} J^{\pi}( x^t_0; c, \widetilde{W} ). 
\end{align}
}
In face of uncertainty, we solve this problem by first utilizing Thompson Sampling \cite{russo2016information} for a $\widetilde{W}_t$ in $\texttt{Ball}_t$ (Eq.~\ref{eq:ball_t}) (Line~\ref{ln:thompson_sampling}) and then computing the safe optimal policy using MPC planning oracle such as model predictive path integral control (MPPI) \cite{williams2017information} under $\widetilde{W}_t$, subject to safety control constraint Eq.~\ref{eq:cbf_learned_approximate_ctrl_constraints} (Line~\ref{ln:safe_mpc}). For computation efficiency in practice, the safe planning problem in Line~\ref{ln:safe_mpc}-\ref{ln:execute} at each episode $t$ could be cast into the following step-wise optimization process:
at each time step $h\in[H]$ in episode $t$, the optimal control sequence to Line~\ref{ln:safe_mpc} without safety consideration could be derived as $\bar{u}^t_{{h:h+H-1}|h}=[\bar{u}^t_{{h}|h},\bar{u}^t_{{h+1}|h},\ldots,\allowbreak \bar{u}^t_{{h+H-1}|h}]$ starting at $x_h^t$ by using MPPI planning algorithm \cite{williams2017information} under dynamics with sampled 
mapping $\widetilde{W}_t$. Then we denote the next control input for execution as nominal control $u^\star\leftarrow \bar{u}^t_{{h}|h}$ and map it to the safe policy class by solving the following quadratic programming (QP) with linear control constraint Eq.~\ref{eq:cbf_learned_approximate_ctrl_constraints}:
\begin{align}\footnotesize
 u_h^t = \argmin_{u} \norm{u-u^\star}^2 \quad 
 \text{s.t.} \quad Eq.~\ref{eq:cbf_learned_approximate_ctrl_constraints} \quad \text{and}\quad u\in[u^-,u^+]\nonumber
\end{align}
Note that the QP may be infeasible when no control solution can be found with the chosen barrier function constraints where the violation of safety is inevitable, e.g. starting at the boundary of safe set and too late to enforce safety. Readers are referred to \cite{squires2018constructive} for discussion on how to create a barrier function that guarantees feasibility of the QP through nominal evading maneuvers.

Then after executing $\pi^t_s(x_h^t):\rightarrow u_h^t$ at $x_h^t$, the process is repeated at $h+1$ on state $x^t_{h+1}$, computes unconstrained nominal control sequence $\bar{u}^t_{{h+1:h+H}|h+1}$, obtains safe control $\pi^t_s(x_{h+1}^t):\rightarrow u_{h+1}^t$ mapped from $\bar{u}_{{h+1}|h+1}^t$ as above, and a trajectory $\tau^t:=\{x_h^t,u_h^t,c_h^t,x_{h+1}^t\}_{h=0}^{H-1}$ is iteratively sampled.


Given Eq.~\ref{eq:double_min} and conditioned on the high probability event that $W^\star \in \texttt{Ball}_t$, and $\pi^\star \in \Pi_{W^\star}$ by definition of $\pi^\star$,  we can show optimism in the sense that:
\begin{align*}\footnotesize
J^{\pi^t}\left( x^t_0; c, W_t \right) \leq J^{\pi^\star}(x_0; c, W^\star). 
\end{align*} 
Hence the optimism allows us to prove the following main statement regarding the bounded regret similar to \cite{kakade2020information}.
\begin{theorem}\label{theorem:regret_bound}[Regret Bound] Set $\lambda = \bar{\sigma}^2 / C_1^2$. Our algorithm learns a sequence of policies $\pi^0,\dots, \pi^{T-1}$ in $T$ episodes, such that in expectation, we have:
{\footnotesize
\begin{align*}
\mathbb{E}\left[ \text{Regret}_T \right] \leq \widetilde{\mathcal{O}}\left( H \sqrt{ H r(r + n + H) T}  \right).
\end{align*} 
}
Also with probability at least $1- O (\delta_s)$, we have that for all $t\in[T], h^s\in[H]$, $h(x_h^t) \geq -\mathcal{O}(L \epsilon / \eta)$.
\end{theorem}
Detailed proof is provided in \wl{Appendix Section~\ref{sec:app:theorem:regret_bound}}.
\vspace{-0.5cm}

\section{Results}\label{sec:results}

We evaluate our Optimism-based Safe Learning framework on two simulation platforms: simulated unicycle mobile robot and inverted pendulum. 
The cumulative rewards (negative cost) are used for evaluations, i.e. the higher the better. 



\begin{figure*}[!ht]
  \centering
  \begin{subfigure}{0.4\textwidth}
\includegraphics[width=\textwidth]{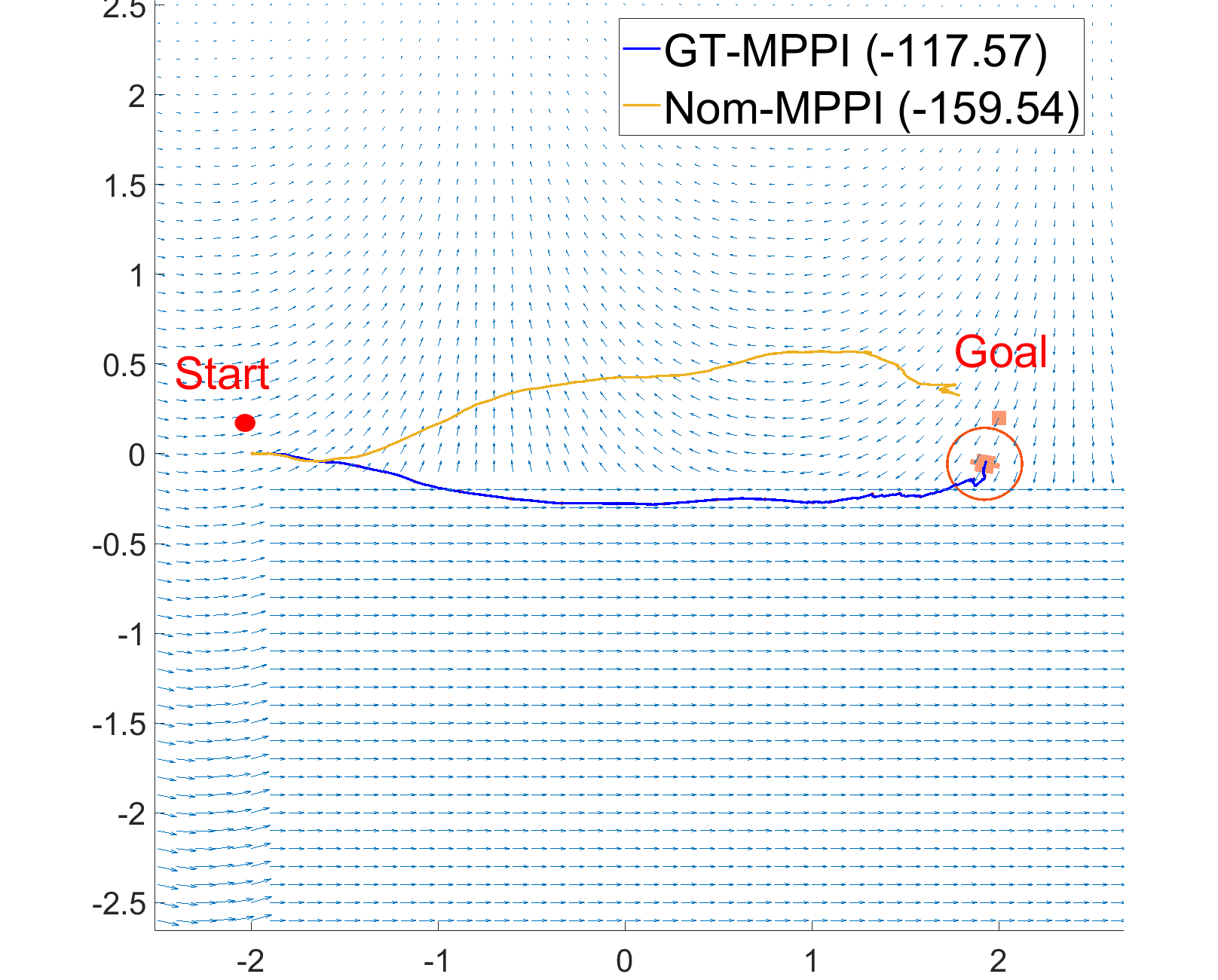}
    \caption{}
    \label{fig:sim2_a}
  \end{subfigure}
    \begin{subfigure}{0.4\textwidth}
\includegraphics[width=\textwidth]{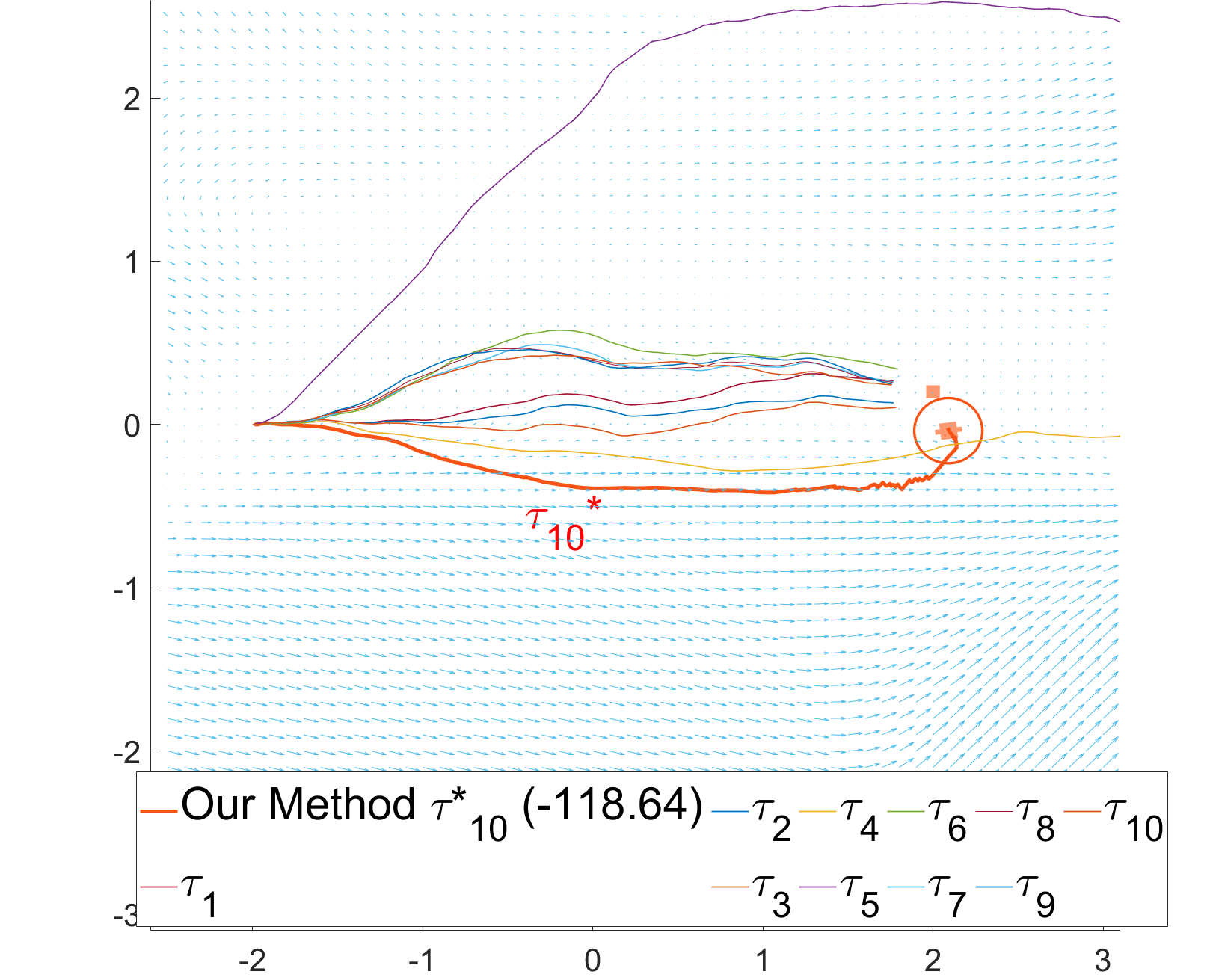}
    \caption{}
    \label{fig:sim2_b}
  \end{subfigure}
  \begin{subfigure}{0.4\textwidth}
\includegraphics[width=\textwidth]{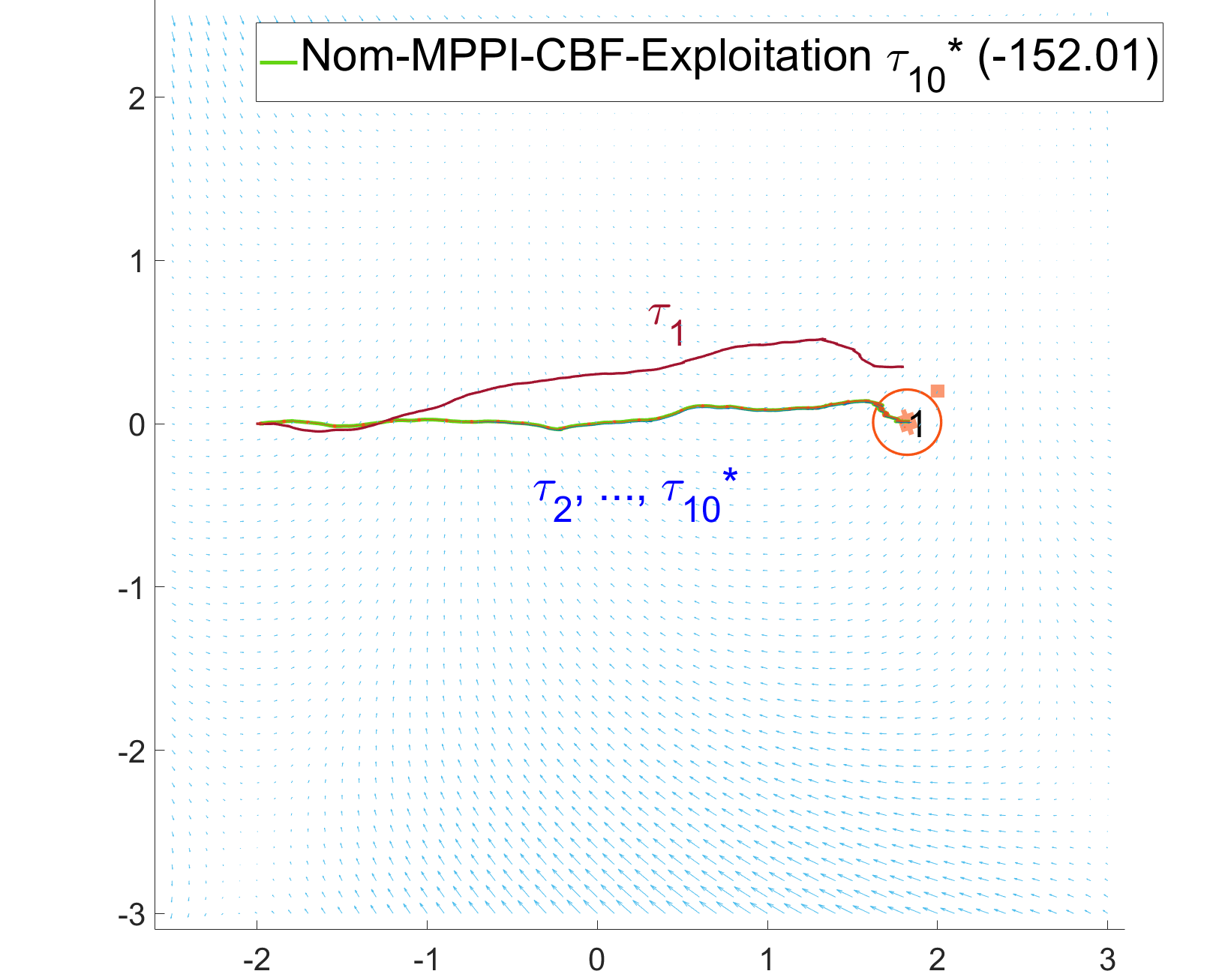}
    \caption{}
    \label{fig:sim2_c}
  \end{subfigure}
\begin{subfigure}{0.4\textwidth}
\includegraphics[width=\textwidth]{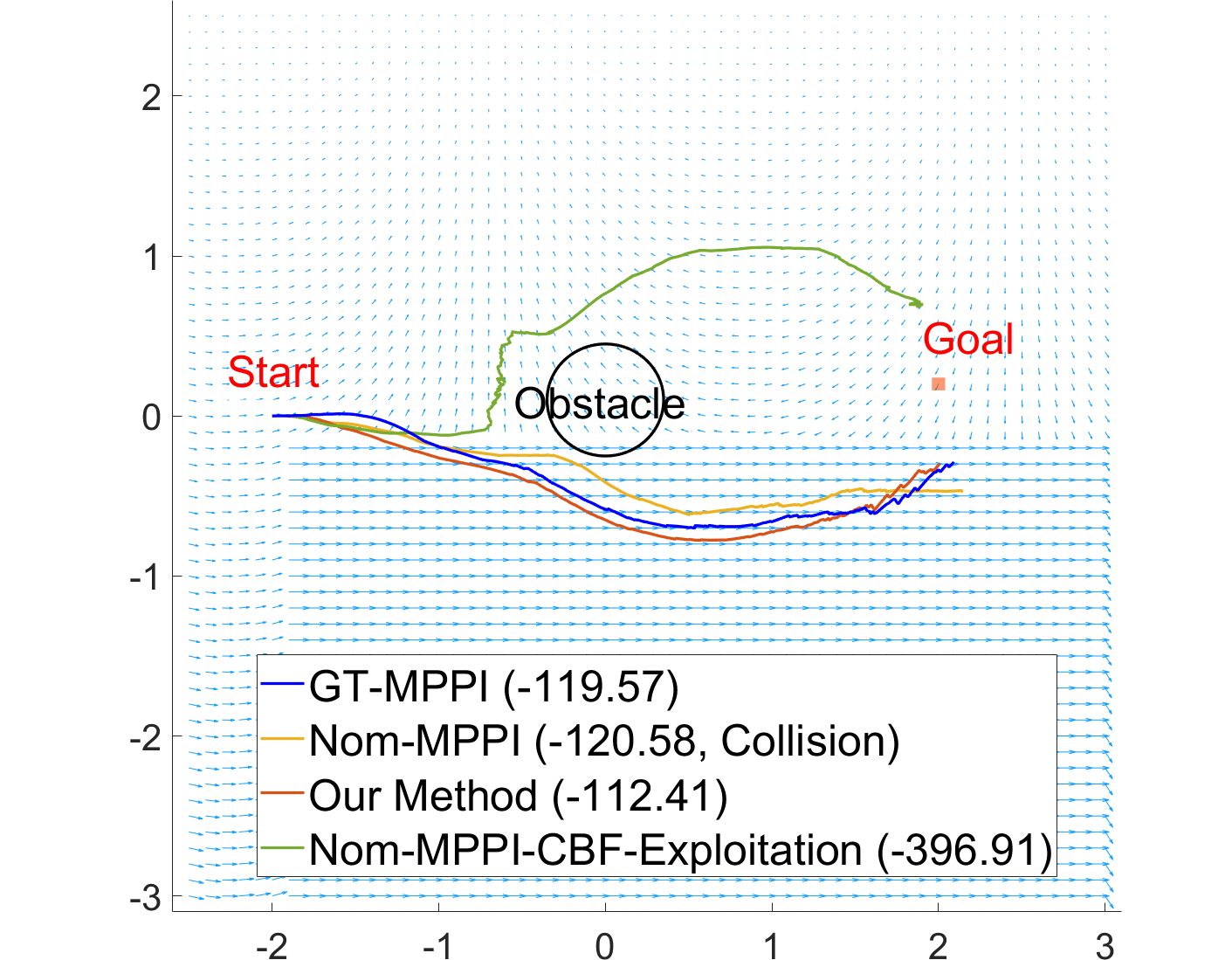}
    \caption{}
    \label{fig:sim2_d}
  \end{subfigure}
\caption{Mobile robot navigation trajectories in an unknown wind field. (a) ground-truth wind field and trajectories from GT-MPPI (rewards: -117.57) and Nom-MPPI (rewards: -159.54) without obstacles, (b) predicted wind field and trajectories from our method (Algorithm~\ref{alg:optmism_based_learning}) during training and testing after 10 episodes (rewards: -118.64) without obstacles, (c) predicted wind field and trajectories from our method using exploitation behavior during training and testing after 10 episodes (rewards: -152.01) without obstacles, (d) ground-truth wind field and trajectories from four different methods during testing with one static obstacle.
}
  \label{fig:sim2}
\end{figure*}\vspace{-0.5cm}

\subsection*{Mobile Robot Navigation}

To compare our method using exploration behavior (Algorithm~\ref{alg:optmism_based_learning}) against exploitation (Nom-MPPI-CBF-Exploitation) in terms of sample efficiency, consider a mobile robot navigation task simulated in Matlab (Figure~\ref{fig:sim2}) where the unicycle robot moves in a wind field that represents the unmodelled part of the robot dynamics. 
In particular,
the algorithm Nom-MPPI-CBF-Exploitation in comparison denotes our method with exploitation behavior only, i.e. replace Line~\ref{ln:thompson_sampling} in Algorithm~\ref{alg:optmism_based_learning} by $\widetilde{W}_t\leftarrow \overline{W}_t$.
Algorithms GT-MPPI and Nom-MPPI denote the MPPI method \cite{williams2017information} with ground-truth dynamics model and nominal dynamics (i.e. $d(\cdot)=0$) respectively.
Here we assume the unicycle dynamics is directly available as the nominal model for the learner but suffers from unknown wind field defined by $d^\star(x):=[\cos(x_1-4)(x_2-3),\; \sin(x_1-4)(x_2-3)]^\top\in\mathbb{R}^2$ with $x=[x_1,x_2]^\top$ as the position of the robot. In particular, there is a rectangle area $[-2,3]\times[-2.6, -0.2]$ in the environment where the wind $d^\star(x)$ has uniform directions (East pointing) with larger magnitude. We use standard quadratic normalized cost $c=(x-x_{goal})^\top Q (x-x_{goal})+u^\top R u$ where $Q,R$ are positive-definite to reflect the cost to go and to learn the optimal policy driving the robot towards the goal.
\begin{figure*}[t]
  \centering
  \begin{subfigure}{0.32\textwidth}
\includegraphics[width=\textwidth]{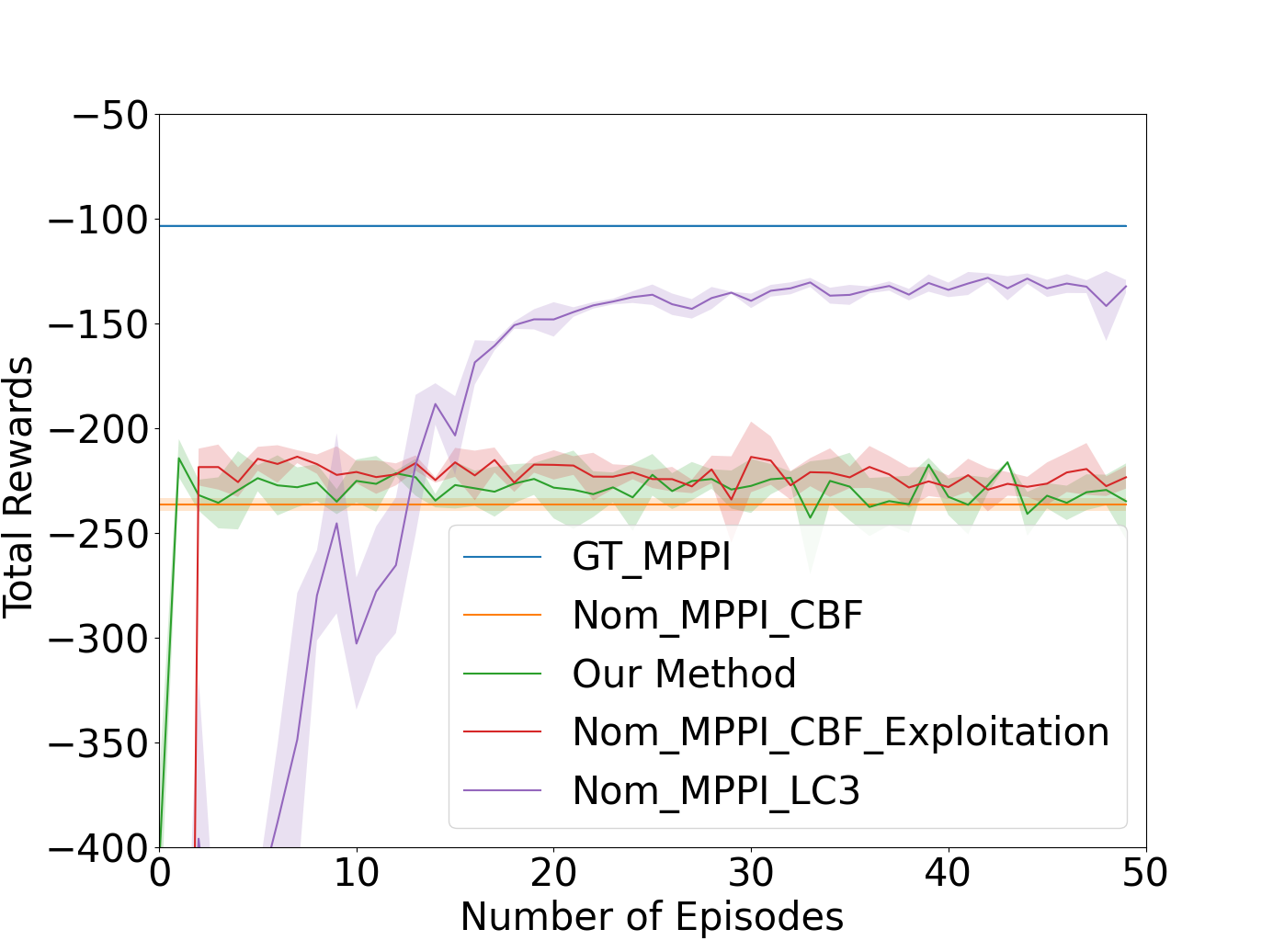}
    \caption{}
    \label{fig:sim1_1_a}
  \end{subfigure}
    \begin{subfigure}{0.32\textwidth}
\includegraphics[width=\textwidth]{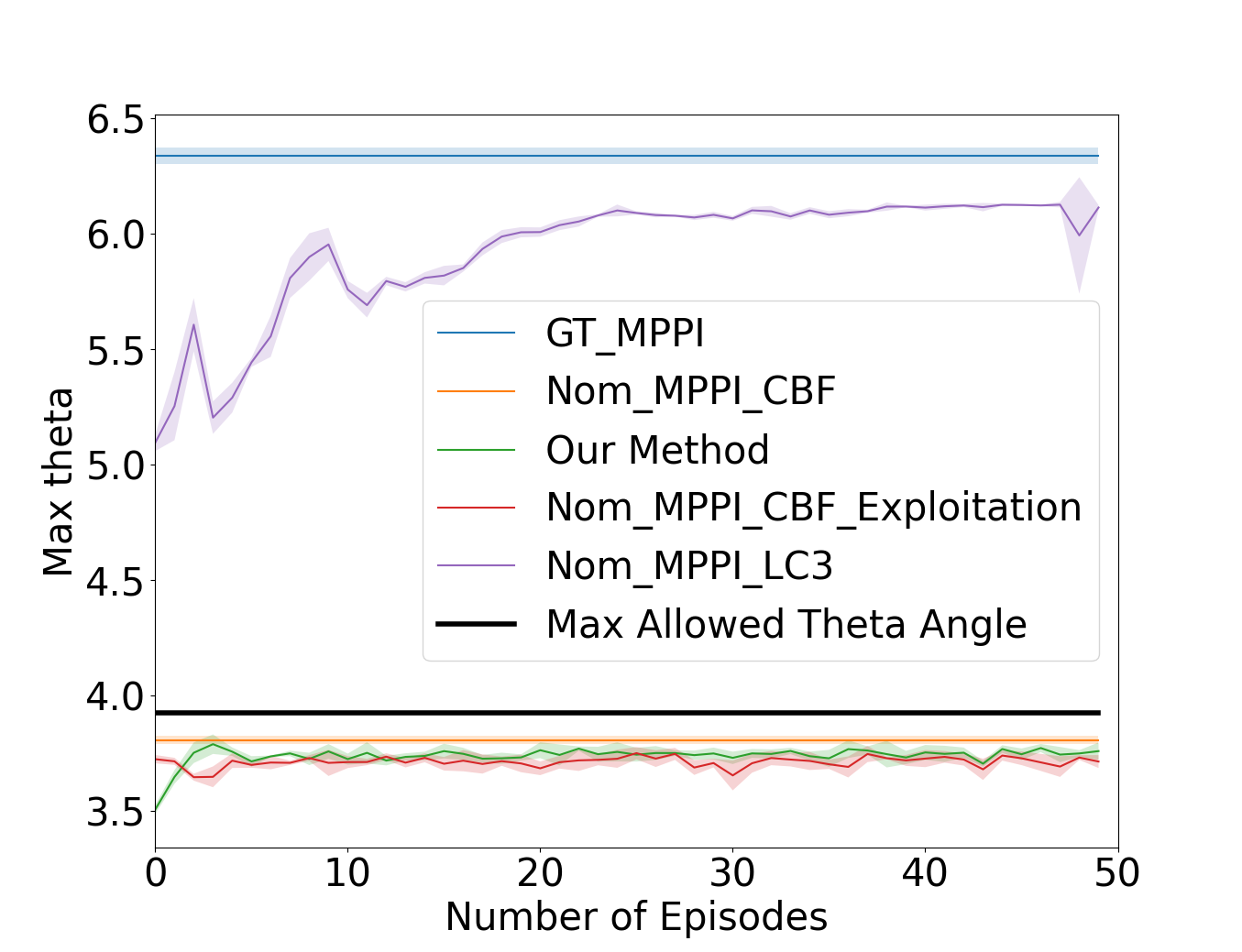}
    \caption{}
    \label{fig:sim1_1_b}
  \end{subfigure}
  \begin{subfigure}{0.32\textwidth}
\includegraphics[width=\textwidth]{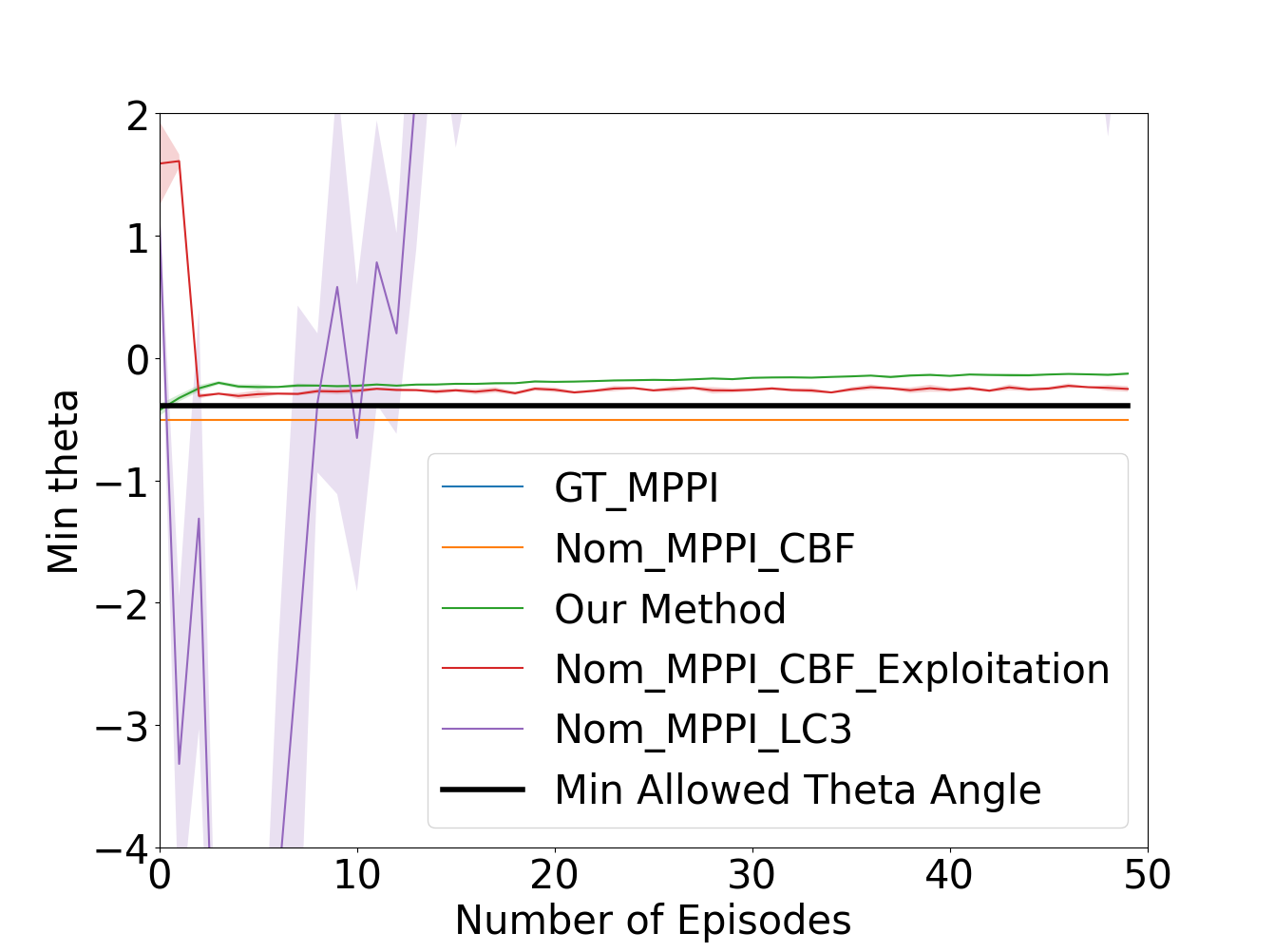}
    \caption{}
    \label{fig:sim1_1_c}
  \end{subfigure}
\caption{Performance curves of (a) cumulative rewards, (b) maximum theta angle, and (c) minimum theta angle in Inverted Pendulum environment testing under the same initial condition.
}
  \label{fig:sim1_1}
\end{figure*}

For obstacle-free scenario shown in Figure~\ref{fig:sim2}(a)-(c), GT-MPPI \cite{williams2017information} with known ground-truth wind model is able to plan the trajectory that takes advantage of the wind field to enjoy the lowest cost (highest reward). It is observed that our method with optimism-based exploration behavior (Algorithm~\ref{alg:optmism_based_learning}) in Figure~\ref{fig:sim2}(b) is able to quickly find a near-optimal trajectory after data collection during training in 10 episodes. The predicted wind field correctly reflects the significant different wind distribution in the rectangle area due to the exploration process. 
This outperforms the exploitation behavior from Nom-MPPI-CBF-Exploitation in Figure~\ref{fig:sim2}(c) that quickly converges to a local minima without much exploration in the unknown wind field, which thus fails to find the different wind field below in the rectangle area.
Figure~\ref{fig:sim2}(d) shows a safe learning scenario with one static obstacle and the corresponding trajectories by the four different algorithms during testing. Our method 
is able to achieve similar performance compared to GT-MPPI and significantly outperforms Nom-MPPI-CBF-Exploitation whose trajectory (green) takes a much longer detour.
This empirically validated the learning and control performance of our method.\vspace{-0.5cm}

\subsection*{Inverted Pendulum}\label{sec:inverted_pendulum}

To further compare the performance with quantitative results, we use the standard inverted pendulum platform modified from the OpenAI gym environment \cite{brockman2016openai} with additive disturbance of $0.05\cos({\theta_t-3})$ on state update.
The pendulum has ground truth mass $m=1$ and length $l=1$, and is controlled by the limited torque input $u\in[-15, 15]$. The standard cost function $c=\theta^2+0.1\dot{\theta}^2+0.001$ is used to learn the optimal policy keeping the pendulum upright (i.e. $\theta=0$). Similar to \cite{cheng2019end}, we randomly set the safe region to be $\theta\in[-1/8\pi, +5/4\pi]$ radians. We define the true system dynamics as $\theta_{t+1} = \theta_t + \dot{\theta}_{t+1} \Delta t+0.05\cos({\theta_t-3})$ and $\dot{\theta}_{t+1}= \dot{\theta}_t + \frac{3g}{2l}\sin{\theta_t}\Delta t + \frac{3}{ml^2}u\Delta t$.
To describe the partially known system dynamics, we assume a nominal model as $\theta_{t+1} = \theta_t + \dot{\theta}_{t+1} \Delta t$ and $\dot{\theta}_{t+1}= \dot{\theta}_t + \frac{3g}{2l'}\sin{\theta_t}\Delta t + \frac{3}{m'l^2}u\Delta t$ with incorrect model parameters $m'=1.8, l'=1.8$ available to the learner (hence $80\%$ error in model parameters).

Using the same and different initial conditions respectively,  Figure~\ref{fig:sim1_1} and Figure~\ref{fig:sim1_2} compare the cumulative rewards, maximum and minimum theta angle achieved during testing after each training episode by using (1) MPPI \cite{williams2017information} with ground-truth dynamics model (GT-MPPI), (2) MPPI with nominal dynamics model and CBF (Nom-MPPI-CBF), (3) our method of optimism-based safe learning (Algorithm~\ref{alg:optmism_based_learning}), (4) our method with exploitation only, i.e. replace Line~\ref{ln:thompson_sampling} in Algorithm~\ref{alg:optmism_based_learning} by $\widetilde{W}_t\leftarrow \overline{W}_t$ (Nom-MPPI-CBF-Exploitation), and (5) unconstrained Lower Confidence-based Continuous Control algorithm (LC3) \cite{kakade2020information}. The last three learning-based algorithms are trained for 50 episodes with 20 testing trials after each training episode averaged from four random seeds. It is observed that our method quickly increased cumulative reward in early stage while satisfying the safety constraints as learning process evolves, and our method using exploration behavior (our method) is able to increase reward even faster than our method using exploitation behavior (Nom-MPPI-CBF-Exploitation), empirically implying sample efficiency. In contrast, GT-MPPI and LC3 severely violate angle limitation due to lack of safety consideration, and safe MPPI using CBF with nominal model (Nom-MPPI-CBF) still violates safety constraints with lower cumulative rewards due to the inaccurate nominal model with large error.


\begin{figure*}[t]
  \centering
  \begin{subfigure}{0.32\textwidth}
\includegraphics[width=\textwidth]{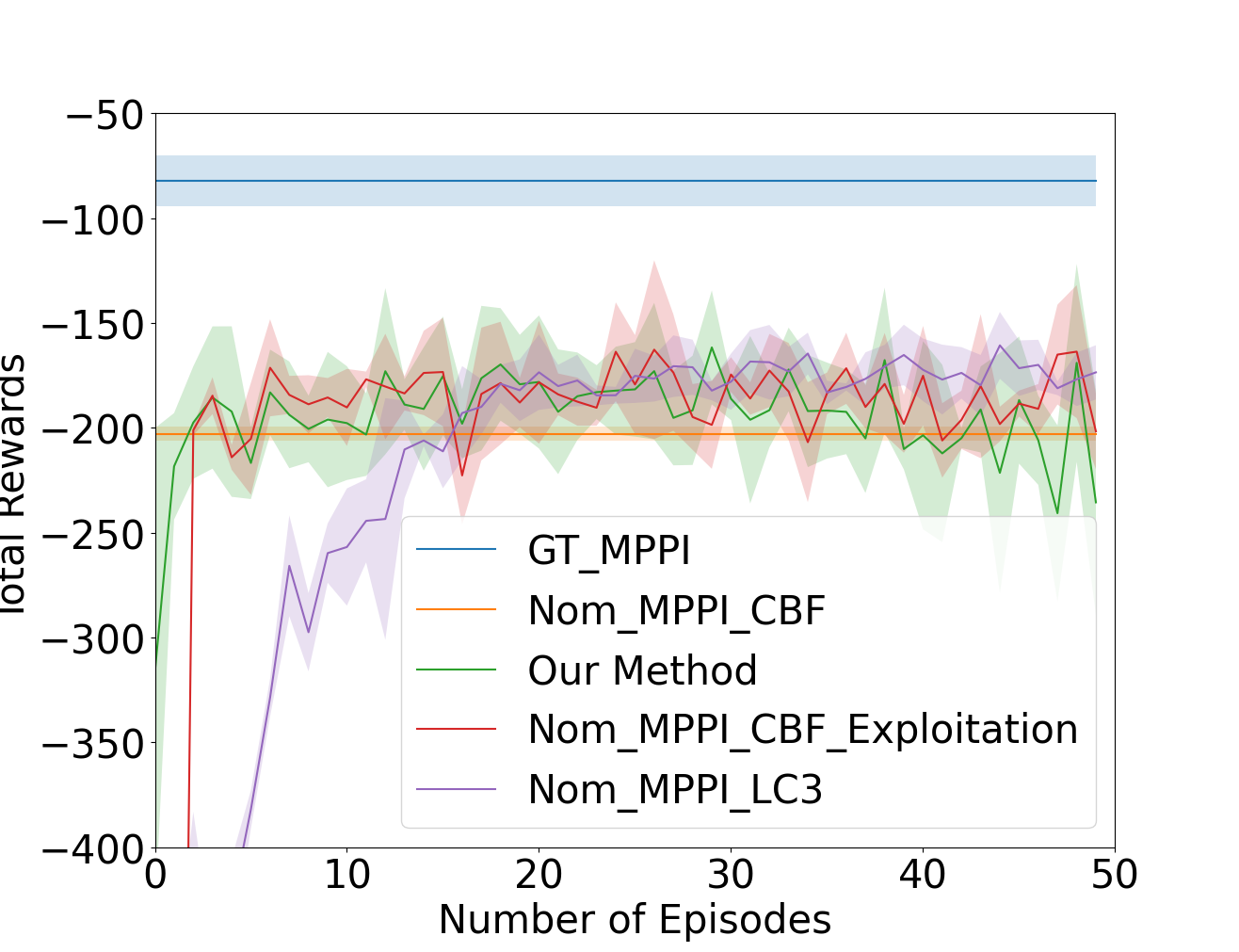}
    \caption{}
    \label{fig:sim1_2_a}
  \end{subfigure}
    \begin{subfigure}{0.32\textwidth}
\includegraphics[width=\textwidth]{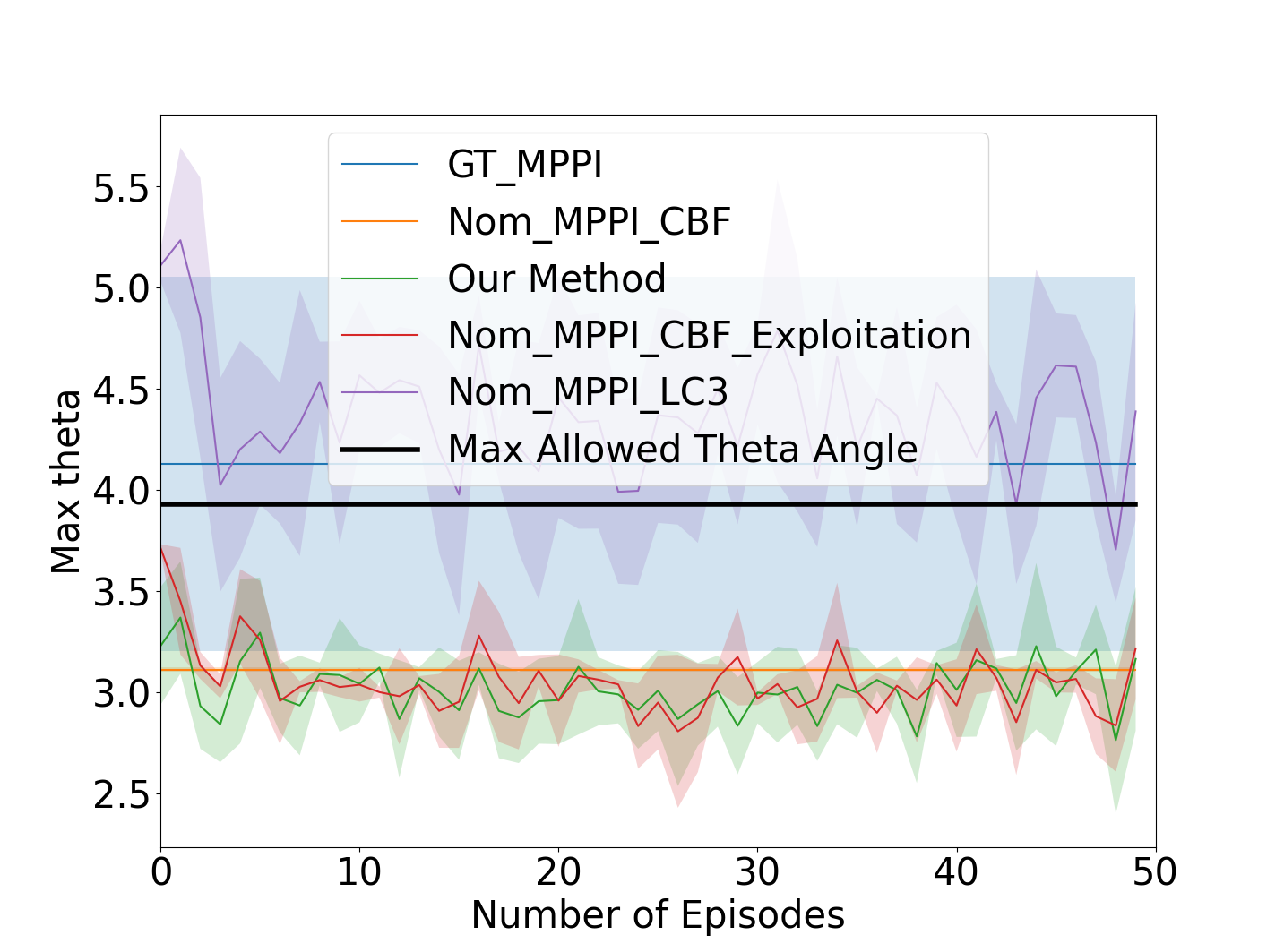}
    \caption{}
    \label{fig:sim1_2_b}
  \end{subfigure}
  \begin{subfigure}{0.32\textwidth}
\includegraphics[width=\textwidth]{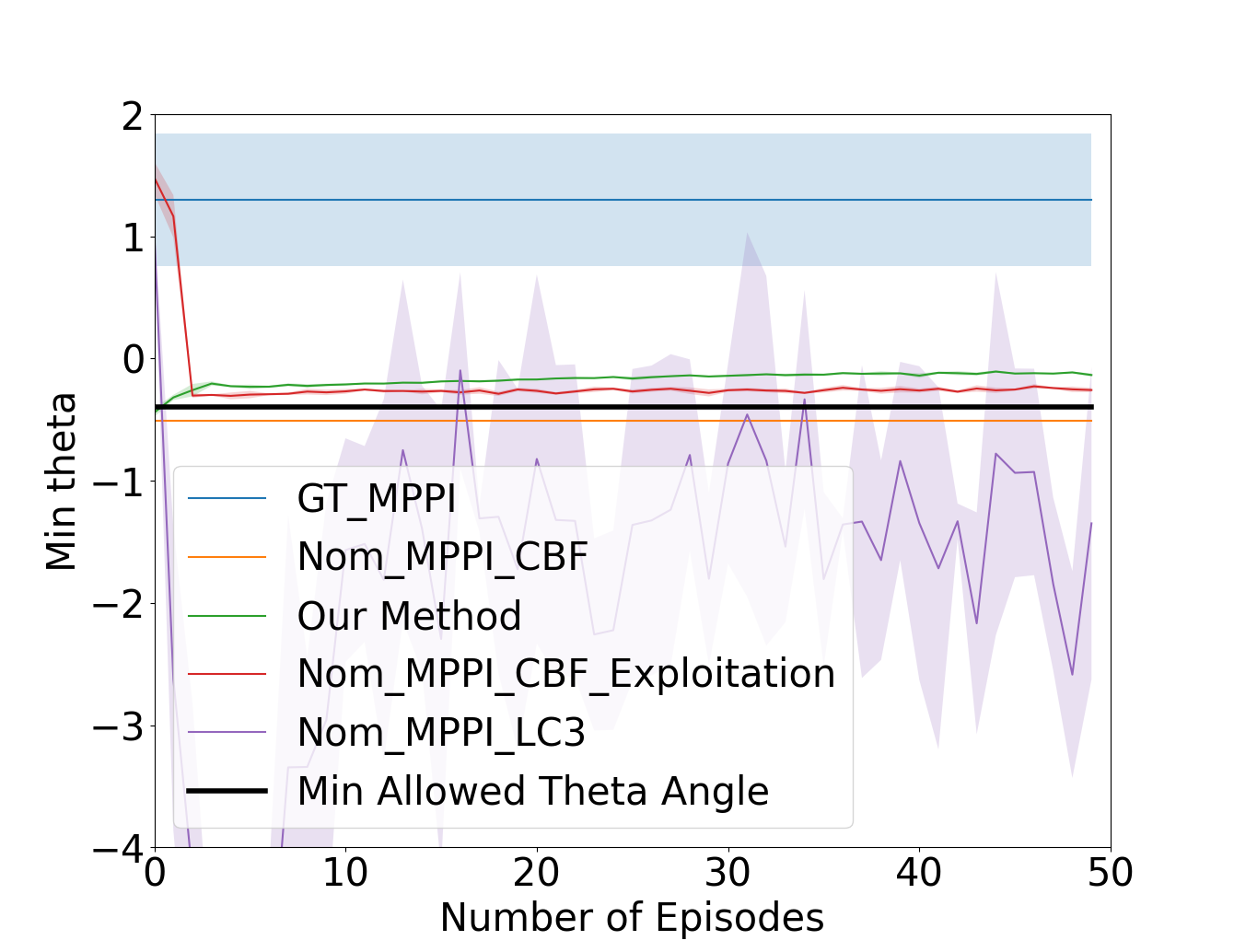}
    \caption{}
    \label{fig:sim1_2_c}
  \end{subfigure}
\caption{Performance curves of (a) cumulative rewards, (b) maximum theta angle, and (c) minimum theta angle in Inverted Pendulum environment with four different initial conditions.
}
  \label{fig:sim1_2}
\end{figure*}\vspace{-0.5cm}

\section{Conclusion} 
\label{sec:conclusion}

In this paper, we address the problem of episodic safe learning for online nonlinear control tasks. Compared to existing safe learning and control approaches that exhaustively expanding safety region or optimizing policy performance without efficiency guarantee, we propose an optimism-based online safe learning framework that simultaneously achieve sample efficient learning for safe behaviors and nonlinear control optimization with bounded regret guarantee. 
We believe our presented work is an important first step to bridge provably efficient learning based methods and model based safety-critical control with formal guarantees. Future work include extending our sample efficient learning on more complex high-dimensional dynamical systems. 

\bibliographystyle{splncs04}
{\footnotesize
\bibliography{ref}
}

\newpage
\appendix
\section*{Appendix 
}

*Equation indexes from (1)-(19) follow the original indexes appearing in the paper submission and new equations start from (20) in this appendix.
\section{Remark~\ref{remark1}}\label{app:remark1}

\begin{customrmk}{1}\label{remark1}
In general for nonlinear function $h^s(\cdot)$ and nonlinear dynamical system, the constraint in Eq.~\ref{eq:constraint_barrier} is nonlinear with respect to the control $u$. 
When both $\hat{f}$ and $d$ are affine control functions in the form of $G_1(x) + G_2(x) u$, the constraint in Eq.~\ref{eq:constraint_barrier} becomes linear with respect to $u$. 
\end{customrmk}
\begin{proof}
Here we discuss how to derive the control constraints with nonlinear control barrier function $h^s(\cdot)$ from Eq.~\ref{eq:ground_truth_opt_const} that fulfills Eq.~\ref{eq:constraint_barrier} (and hence fulfills Proposition~\ref{prop:forward_invariant}). 
Recall the constraint Eq.~\ref{eq:ground_truth_opt_const} as follows (we have $u_h=\pi(x_h)$).
\begin{align}\footnotesize\label{eq:app:discrete_const}
\begin{split}
h^s\left( \hat{f}(x_h,u_h) + d(x_h,u_h) \right) - L\bar{\sigma} \sqrt{ 2n \ln\left( \frac{H n}{\delta_s}\right)}
 - h^s(x_h) \\
 \geq  - \eta h^s(x_h)
\end{split}
\end{align}
Given that both the known nominal discrete dynamics $\hat{f}$ and the unknown part $d$ are affine in control as $\hat{f}(x_h, u_h)=\hat{F}(x_h)+\hat{G}(x_h)u_h$ and $d(x_h,u_h)=g_1(x_h)+g_2(x_h)u_h$, where $\hat{F},\hat{G},g_1,g_2$ are assumed locally Lipschitz continuous. Then with the continuously differentiable function $h^s(\cdot)$, we have $h^s\left( \hat{f}(x_h,u_h) + d(x_h,u_h) \right) - h^s(x_h) = L_{\hat{F}+g_1}^{\Delta}h^s(x_h)+L_{\hat{G}+g_2}^{\Delta}h^s(x_h)u_h$ and hence Eq.\ref{eq:app:discrete_const} can be re-written as
\begin{align}\footnotesize\label{eq:app:ctrl_constraint}
L_{\hat{F}+g_1}^{\Delta}h^s(x_h)+L_{\hat{G}+g_2}^{\Delta}h^s(x_h)u_h\geq - \eta h^s(x_h) +  L\bar{\sigma} \sqrt{ 2n \ln\left( \frac{H n}{\delta_s}\right)}
\end{align}
where $L_{\hat{F}+g_1}^{\Delta}h^s(x_h)$ and $L_{\hat{G}+g_2}^{\Delta}h^s(x_h)$ are discrete-time Lie-derivatives of $h^s(x_h)$ obtained through Taylor's theorem along $\hat{F}(x_h)+g_1(x_h)$ and $\hat{G}(x_h)+g_2(x_h)$ respectively. To that end, the condition in Eq.~\ref{eq:constraint_barrier} and Eq.~\ref{eq:ground_truth_opt_const} hold by enforcing the linear control constraint Eq.~\ref{eq:app:ctrl_constraint} on $u_h$.
Thus, we conclude the proof. 
\end{proof}

\section{Assumption~\ref{lem:valid_initialization}}\label{sec:app:lem:valid_initialization}
\begin{customassu}{2}
\label{app:lem:valid_initialization} (Calibrated model) With $\overline{W}_0, V_0$ from the initial data $(x_i, u_i, x_i')_{i=1}^N$ and $\epsilon,\delta\in(0,1)$, we can build the initial confidence ball describing the uncertain region of the linear mapping $W^\star$ with probability at least $1-\delta$ as follows:
\begin{align}\label{app:eq:ball_0}\footnotesize
\texttt{Ball}_0 = \left\{ W: \left\| (W - \overline{W}_0) V_0^{1/2}  \right\|_2 \leq \beta, \quad \| W  \|_2 \leq \|W^\star\|_2     \right\}\tag{10}
\end{align} where $\beta$
is a hyper-parameter describing an appropriate confidence radius.
Then for all  $\widetilde{W} \in \texttt{Ball}_0$, we have:
\begin{align*}\footnotesize
\forall x,u\in\mathcal{X}\times\mathcal{U}: \; \left\| \left(\widetilde{W} - W^\star\right) \phi(x,u) \right\|_2 \leq \mathcal{O}\left( \epsilon \right).
\end{align*}
\end{customassu}

We introduce the following Lemma~\ref{lem:pre-train} to provide a practical example on how to derive such calibrated model in Eq.~\ref{app:eq:ball_0}.

\begin{lemma}\label{lem:pre-train}
[Pre-train guarantee of calibrated model from pre-collected data] 
Fix a pair $(\epsilon,\delta)$ with $\epsilon,\delta \in (0,1)$. Denote $\Phi\in\mathbb{R}^{r\times N}$ where each column of $\Phi$ corresponds to the feature vector $\phi(x,u)$ for $(x,u) \in \mathcal{X}\times\mathcal{U}$. Assume $\text{span}(\Phi) = r$. Via John's theorem, denote $\mathcal{B} \subset \mathcal{X}\times\mathcal{U}$ as the core set of John's ellipsoid, and $\mu$ as the corresponding sampling distribution with support on $\mathcal{B}$ defined by $\mu = \argmax_{ \rho\in\Delta(\mathcal{X}\times\mathcal{U}) } \ln\det\left( \mathbb{E}_{x,u\sim \rho}\phi(x,u)\phi(x,u)^{\top} \right)$ from John's ellipsoid. Then draw $N$ triples $\mathcal{D} = \{x_i,u_i,x'_i\}_{i=1}^N$ as pre-collected offline dataset with $x_i,u_i\sim \mu, x_i'\sim P(\cdot | x_i,u_i)$, and compute the initialization $\overline{W}_0 = \sum_{i=1}^N ( x_i' - \hat{f}(x_i,u_i) ) \phi(x_i,u_i)^{\top} V_0^{-1}$ with $V_0 =  \sum_{i=1}^N \phi(x_i,u_i)\phi(x_i,u_i)^{\top} + \lambda I$. Then with probability at least $1-\delta$, we have:
\begin{align*}
\forall x,u\in\mathcal{X}\times\mathcal{U}, \quad \left\| (\overline{W}_0 - W^\star)\phi(x,u)  \right\|_2 \leq O(\epsilon), 
\end{align*} with polynomially number of samples, i.e., $N$ scaling polynomially with respect to the relevant parameters:
\begin{equation*}\footnotesize
N = \mathcal{O}\Big( \frac{r C_1^2 \lambda + r {\bar{\sigma}}^2 n + \ln(1/\delta) + r^2 {\bar{\sigma}}^2}{\epsilon^2} + \frac{C_1^2 r^2 \ln( r / \delta) }{\epsilon^4}  \Big)
\end{equation*}
After deriving $\overline{W}_0, V_0$ from the initial data $(x_i, u_i, x_i')_{i=1}^N$, we can build the initial confidence ball describing the uncertain region of $W^\star$ as follows:
\begin{align}
\texttt{Ball}_0 = \left\{ W: \left\| (W - \overline{W}_0) V_0^{1/2}  \right\|_2 \leq \beta, \quad \| W  \|_2 \leq \|W^\star\|_2     \right\}\tag{10}
\end{align} where $\beta $ is the confidence radius as $\beta := \sqrt{\lambda} C_1 + \bar{\sigma} \sqrt{ 8 n \ln(5) + 8 r \ln\left( 1 + N/\lambda  \right)  + 8 \ln(1/\delta)}$. For all  $\widetilde{W} \in \texttt{Ball}_0$, we also have
{\footnotesize
\begin{align*}\footnotesize
\forall x,u\in\mathcal{X}\times\mathcal{U}: \; \left\| \left(\widetilde{W} - W^\star\right) \phi(x,u) \right\|_2 \leq \mathcal{O}\left( \epsilon \right).
\end{align*}
}  
\end{lemma}
\begin{proof}
First note that we can compute the exact difference between the least square solution $\overline{W}_0$ and $W^\star$:
\begin{align*}
\overline{W}_0 - W^\star  & = -\lambda W^\star \left( V_0\right)^{-1} + \sum_{i=1}^N \epsilon_i \phi(x_i,u_i)^{\top} V_0^{-1}.
\end{align*} 
Continue, we have
\begin{align*}\small
\left\| (\overline{W}_0 - W^\star) V_0^{1/2} \right\|_2 
&\leq \left\| \lambda W^\star V_0^{-1/2}   \right\|_2 + \left\| \sum_{i=1}^N \epsilon_i \phi(x_i,u_i)^{\top} V_{0}^{-1/2} \right\|_2 \\
& \leq \sqrt{\lambda}  C_1 + \bar{\sigma} \sqrt{ 8 n \ln(5) + 8 \ln\left( \det(1 + V_0 / \lambda)  \right)  + 8 \ln(1/\delta)} \\
& \leq  \underbrace{\sqrt{\lambda}  C_1 + \bar{\sigma} \sqrt{ 8 n \ln(5) + 8 r \ln\left( 1 + N/\lambda  \right)  + 8 \ln(1/\delta)}}_{:=\beta}
\end{align*}
where $C_1$ denotes the standard assumption of bounded norm $\|W^\star\|_2 \leq C_1$.
Denote $\Sigma = \mathbb{E}_{x,u\sim \mu} \phi(x,u)\phi(x,u)^{\top}$. Via matrix Bernstein's inequality, we get that with probability at least $1-\delta$, for any $x$ with $\|x\|_2 \leq 1$,
\begin{align*}
\left\lvert x^{\top} \left(  \sum_{i=1}^N \phi(x_i,u_i)\phi(x_i,u_i)^{\top}/N - \Sigma \right) x \right\rvert \leq  \frac{2\ln( 8 r / \delta) }{3N} + \sqrt{\frac{2\ln(8 r / \delta)}{N}} := \varepsilon.
\end{align*}
Thus we will have that for any $x$ with $\|x\|_2 \leq 1$
:
\begin{align*}
x^{\top} (\overline{W}_0 - W^\star) V_0 (\overline{W}_0 - W^\star)^{\top} x \geq  x^{\top} (\overline{W}_0 - W^\star) ( \Sigma N + \lambda ) (\overline{W}_0 - W^\star)^{\top} x -  2 \varepsilon N C_1,
\end{align*} 
which means that:
\begin{align*}
\left\| (\overline{W}_0 - W^\star) ( \Sigma + \lambda/N )^{1/2}  \right\|^2_2 \leq & \beta^2 /N + 2C_1 \varepsilon \\
\leq & \frac{\lambda C_1^2}{ N} + \frac{\bar{\sigma}^2 ( n + r\ln(1+N/\lambda + \ln(1/\delta))}{N} + \frac{2C_1 \sqrt{\ln(8r/\delta)}}{\sqrt{N}}
\end{align*}
For any $x,u$, we have:
\begin{align*}
\left\lvert (\overline{W}_0 - W^\star) \phi(x,u) \right\rvert^2 
\leq  \left\| (\overline{W}_0 - W^\star) ( \Sigma + \lambda/N )^{1/2}  \right\|^2_2 \left\| ( \Sigma + \lambda/N )^{-1/2} \phi(x,u)  \right\|_2^2
\end{align*}
Note that for any $x$, we have:
\begin{align*}
x^{\top}  \Sigma^{-1} x \geq x^{\top} (\Sigma + \lambda / N)^{-1} x.
\end{align*} Using the John's theorem, we get that:
\begin{align*}
\phi(x,u)^{\top} (\Sigma + \lambda / N)^{-1} \phi(x,u) \leq  \phi(x,u)^{\top} \Sigma^{-1} \phi(x,u) \leq r
\end{align*}
Hence, we have:
\begin{align*}
\left\lvert (\overline{W}_0 - W^\star) \phi(x,u) \right\rvert 
\leq &  \sqrt{ \left(\frac{\beta^2}{N} + 2C_1 \varepsilon \right) r} \\
\leq &  \sqrt{ \frac{r \lambda C_1^2 }{ N}} + \sqrt{\frac{ r\bar{\sigma}^2 ( n + r\ln(1+N/\lambda + \ln(1/\delta))}{N} }
 + \sqrt{ \frac{2C_1 r \sqrt{\ln(8r / \delta)}}{\sqrt{N}} }
\end{align*}
Now setting $N = \mathcal{O}\Big( \frac{r C_1^2 \lambda + r {\bar{\sigma}}^2 n + \ln(1/\delta) + r^2 {\bar{\sigma}}^2}{\epsilon^2} + \frac{C_1^2 r^2 \ln( r / \delta) }{\epsilon^4}  \Big)$, we ensure that:
\begin{align*}
\left\lvert (\overline{W}_0 - W^\star) \phi(x,u) \right\rvert \leq O(\epsilon).
\end{align*} 
Then 
starting from triangle inequality, we get:
\begin{align*}
\left\lvert (\widetilde{W} - W^\star ) \phi(x,u)\right\rvert & \leq \left\lvert (\widetilde{W} - \overline{W}_0 ) \phi(x,u)\right\rvert  + \left\lvert (\overline{W}_0 - W^\star ) \phi(x,u)\right\rvert  \\
& \leq \left\| (\widetilde{W} - \overline{W}_0) (\Sigma + \lambda / N)^{1/2}\right\|_2 \left\| (\Sigma + \lambda / N)^{-1/2}\phi(x,u) \right\|_2 \\
& \quad + \left\| (\overline{W}_0 - {W}^\star) (\Sigma + \lambda / N)^{1/2}\right\|_2 \left\| (\Sigma + \lambda / N)^{-1/2}\phi(x,u) \right\|_2 \\
& \leq  \left\| (\widetilde{W} - \overline{W}_0) (\Sigma + \lambda / N)^{1/2}\right\|_2 \sqrt{r} + \left\| (\overline{W}_0 - {W}^\star) (\Sigma + \lambda / N)^{1/2}\right\|_2 \sqrt{r}
\end{align*}
We also know that for any two ${W}_1$ and $W_0$ with $\|{W}_j\|_2 \leq C_1$ with $j\in \{1,2\}$, we have:
\begin{align*}
x^{\top} ({W}_1 - {W}_2) V_0 ({W}_1 - {W}_2)^{\top} x \geq x^{\top} ({W}_1 - {W}_2) ( \Sigma N + \lambda ) ({W}_1 - {W}_2)^{\top} x -  2 \varepsilon N C_1,
\end{align*} which means that:
\begin{align*}
&\left\| (\overline{W}_0 - \widetilde{W}) ( \Sigma + \lambda/N )^{1/2}  \right\|^2_2  \leq  \beta^2 /N + 2C_1 \varepsilon,\\
&\left\| (\overline{W}_0 - {W}^\star) (\Sigma + \lambda / N)^{1/2}\right\|_2 \leq \beta^2 /N + 2C_1 \varepsilon.
\end{align*}
This implies that:
\begin{align*}
\left\lvert (\widetilde{W} - W^\star ) \phi(x,u)\right\rvert \leq 2 \sqrt{r} \sqrt{ \beta^2 / N + 2C_1 \varepsilon }.
\end{align*}
Now recall the setup of $N$, $\beta$, and $\varepsilon$, we conclude the proof.
\end{proof}


As the typical assumption similar to \cite{berkenkamp2017safe}, Assumption~\ref{app:lem:valid_initialization} represents an initially calibrated model $\overline{W}_0$, whose initial confidence region $\texttt{Ball}_0$ in Eq.~\ref{app:eq:ball_0} could yield approximately good prediction for all $\widetilde{W} \in \texttt{Ball}_0$.

\section{Proof of Theorem~\ref{prop:approximate_safe}}\label{sec:app:prop:approximate_safe}

\begin{customthm}{1}[Policy for Approximate High-Probability Safety Guarantee with Learned Dynamics] \label{app:prop:approximate_safe}
Under Assumption~\ref{app:lem:valid_initialization}, 
consider any $\widetilde{W}\in\texttt{Ball}_0$, and define any policy $\pi_{s}:\mathcal{X}\mapsto\mathcal{U}$ that satisfies the CBF constraint parameterized by $\widetilde{W}$, i.e., 
\begin{multline}\label{app:eq:cbf_with_approximate_model2}\footnotesize
\forall x \in \mathcal{X}: \pi_s(x) \in \mathcal{U}_{x} := \bigg\{u:  h^s\left( \hat{f}(x,u) + \widetilde{W}\phi(x,u) \right) -\\ 
L{\bar{\sigma}} \sqrt{ 2n \ln\left( \frac{H n}{\delta_s}\right)}  \geq  (1- \eta)    h^s(x) \bigg\}\tag{11}
\end{multline}
Then with probability at least $1-\delta_s$, starting at any safe initial state $h^s(x_0)\geq 0$, $\pi_s$ generates a safe trajectory $\{x_0,u_0,\dots, x_{H-1}, u_{H-1}\}$, such that for all time steps $h\in [H]$, $h^s(x_h) \geq -\mathcal{O}( \frac{L \epsilon}{\eta} )$, where $L$ is the Lipschitz constant of $h^s(\cdot)$ under bounded $x\in\mathcal{X}$.
\end{customthm}
\begin{proof}
Starting from Assumption~\ref{app:lem:valid_initialization}, we know that for any $\widetilde{W}\in\texttt{Ball}_0$, we have:
\begin{align*}
\left\| \left( \widetilde{W} - W^\star \right)\phi(x,u)  \right\|_2 \leq \mathcal{O}(\epsilon), \forall x,u\in\mathcal{X}\times\mathcal{U}.
\end{align*} 
From Eq.~\ref{app:eq:cbf_with_approximate_model2} the policy selects action $u_h$ for all time steps $h\in[H]$ such that:
\begin{equation*}
\begin{split}
h^s( \hat{f}(x_h, u_h)  + \widetilde{W}\phi(x_h,u_h)  )  - L  \bar{\sigma} \sqrt{2 n \ln\left( \frac{H n}{\delta_s}\right) } 
\geq (1- \eta) h^s(x_h)
\end{split}
\end{equation*}
This means that for $W^\star$, we have:
\begin{equation}\label{app:eq:traj_forward_invariant}
    \begin{split}
        h^s(x_{h+1}) &= h^s( \hat{f}(x_h, u_h)  + {W}^\star\phi(x_h,u_h)  +\epsilon_h) \\
& \geq h^s( \hat{f}(x_h, u_h)  + \widetilde{W}\phi(x_h,u_h) ) \\ 
&\quad - L \left\| (\widetilde{W} - W^\star)\phi(x_h,u_h) \right\|_2 - L \|\epsilon_h\|_2 \\
& \geq (1-\eta) h^s(x_h) +  L  \bar{\sigma} \sqrt{2 n \ln\left( \frac{H n}{\delta_s}\right) }\\
&\quad - L \epsilon  - L\|\epsilon_h\|_2 \\
&\geq (1-\eta) h^s(x_h) - L \epsilon \\
&\geq (1-\eta)^2 h^s(x_{h-1}) - L\left(\epsilon + (1-\eta) \epsilon\right) \\
& \geq (1-\eta)^{h+1} h^s(x_0)  - \frac{L}{\eta} \epsilon
    \end{split}
\end{equation}
Using the initial condition that $h^s(x_0) \geq 0$, we conclude the proof.
\end{proof}
Despite the unbounded stochasticity of the dynamics, Eq.~\ref{app:eq:traj_forward_invariant} with $h^s(x_{h+1})\geq (1-\eta)^{h+1} h^s(x_0)  - \frac{L}{\eta} \epsilon$ ensures that for all time steps $h\in [H]$, $h^s(x_{h})$ is always lower bounded with a high probability, implying the probabilistic safety guarantee for the entire trajectory generated under $\pi_{s}$ in Eq.~\ref{app:eq:cbf_with_approximate_model2}.


\section{Proof of Proposition~\ref{prop:intersect_confidence_interval}}\label{sec:app:prop:intersect_confidence_interval}

\begin{customprop}{2}\label{app:prop:intersect_confidence_interval}
Given the uncertainty regions $W^\star\in \{W: \|(W - \overline{W}_t) \Sigma_t^{1/2} \|_2 \leq \beta_t\}$ (Proof of Lemma B.5 in \cite{kakade2020information}) and $\texttt{Ball}_0$ (Eq.~\ref{eq:ball_0}) with the probability of $\text{Pr}(W^\star\in \{W: \|(W - \overline{W}_t) \Sigma_t^{1/2} \|_2 \leq \beta_t\})\geq 1-\delta$ and $\text{Pr}(W^\star\in \texttt{Ball}_0)\geq 1-\delta$, then for all $t$ we have
\begin{multline}\footnotesize
    \text{Pr}\left( W^\star \in \texttt{Ball}_t:=  \texttt{Ball}_0 \cap \left\{ W: \| (W - \overline{W}_t) \Sigma_t^{1/2} \|_2 \leq \beta_t\right\} \right) \\
    \geq 1-2\delta \tag{16}
\end{multline}
\end{customprop}
where $\text{Pr}(\cdot)$ denotes the probability of an event.
\begin{proof}
By definition,
\begin{equation*}
\begin{split}
\text{Pr}\left(W^\star\notin \{W: \|(W - \overline{W}_t) \Sigma_t^{1/2} \|_2 \leq \beta_t\}\right)&\leq \delta \\
\text{Pr}\left(W^\star\notin \texttt{Ball}_0\right)&\leq \delta
\end{split}
\end{equation*}
Thus, we have
\begin{equation*}
\begin{split}
&\text{Pr}\left( W^\star \in \texttt{Ball}_t:=  \texttt{Ball}_0 \cap \left\{ W: \| (W - \overline{W}_t) \Sigma_t^{1/2} \|_2 \leq \beta_t\right\} \right) \\
=& 1 - \text{Pr}\left( W^\star\notin \{W: \|(W - \overline{W}_t) \Sigma_t^{1/2} \|_2 \leq \beta_t\}\;\text{Or}\;W^\star\notin \texttt{Ball}_0  \right)\\
\geq& 1-2\delta
\end{split}
\end{equation*}
which concludes the proof.
\end{proof}

\section{Proof of Eq.~\ref{app:eq:cbf_learned_approximate_ctrl_constraints} $ \Rightarrow$ Eq.~\ref{app:eq:safe_policy_class}}\label{sec:app:eq:safe_policy_class}
Here we show that for all state $x\in \mathcal{X}$, any $u\in\mathcal{U}$ satisfying Eq.~\ref{app:eq:cbf_learned_approximate_ctrl_constraints} ensures $u\in \pi_s(x)$ defined in Eq.~\ref{app:eq:safe_policy_class}, thus constructing the safe policy class $\Pi_{\widetilde{W}}$.

Recall the definition of $\Pi_{\widetilde{W}}$ as follows.
\begin{multline}\footnotesize\label{app:eq:safe_policy_class}
\Pi_{\widetilde{W}} = \bigg\{ \pi_s\in \Pi : \forall x\in\mathcal{X},\pi_s(x) \in \Big\{u:  h^s\left( \hat{f}(x,u) + \widetilde{W}\phi(x,u) \right) \\
- L\bar{\sigma} \sqrt{ 2n \ln\left( \frac{H n}{\delta_s}\right)} \geq  (1- \eta) h^s(x) \Big\}\bigg\}\tag{17}
\end{multline}

Consider Eq.~\ref{app:eq:cbf_learned_approximate_ctrl_constraints}:
\begin{multline}\footnotesize\label{app:eq:cbf_learned_approximate_ctrl_constraints}
u\in \mathcal{U} : 
L_{\hat{F}}^{\Delta}h^s(x)+L_{\hat{G}}^{\Delta}h^s(x)u 
-  L\bar{\sigma} \sqrt{ 2n \ln\left( \frac{H n}{\delta_s}\right)}\\
\geq - \eta h^s(x) 
+ \underbrace{|\Delta h^s(x)\widetilde{W}\phi(x,u^\star)|+|\Delta h^s(x)\widetilde{W}L_{x,\phi}(u^+-u^-)|}_{K(x,u^\star)} \tag{18}
\end{multline}
where $L_{\hat{F}}^{\Delta}h^s(x)$ and $L_{\hat{G}}^{\Delta}h^s(x)$ are discrete-time Lie-derivatives of $h^s(x)$ obtained through Taylor's theorem along $\hat{F}(x)$ and $\hat{G}(x)$ respectively. 
$L_{x,\phi}$ is the local Lipschitz constant vector for the known feature mapping $\phi$ w.r.t. $u$ at $x$.
$\Delta h^s(x)$ is the discrete derivative of $h^s$ and $u^\star, u^+,u^-$ are the nominal, max and min value of $u$ respectively.
Thus we have
\begin{align}
K(x,u^\star) &= |\Delta h^s(x)\widetilde{W}\phi(x,u^\star)|+|\Delta h^s(x)\widetilde{W}L_{x,\phi}(u^+-u^-)|\nonumber\\
&\geq  |\Delta h^s(x)\widetilde{W}\phi(x,u^\star)|+|\Delta h^s(x)\widetilde{W}L_{x,\phi}(u-u^\star)|\nonumber\\
&\geq \left|\Delta h^s(x)\widetilde{W}\phi(x,u^\star)+ \Delta h^s(x)\widetilde{W}L_{x,\phi}(u-u^\star)  \right|\nonumber\\
&\geq -\Delta h^s(x)\widetilde{W}\phi(x,u)
\end{align}
Then with $K(x,u^\star)\geq -\Delta h^s(x)\widetilde{W}\phi(x,u)$, from Eq.~\ref{app:eq:cbf_learned_approximate_ctrl_constraints} we have
\begin{align}\footnotesize\label{app:eq:cbf_learned_approximate_ctrl_constraints2}
u\in \mathcal{U} : 
&L_{\hat{F}}^{\Delta}h^s(x)+L_{\hat{G}}^{\Delta}h^s(x)u 
-  L\bar{\sigma} \sqrt{ 2n \ln\left( \frac{H n}{\delta_s}\right)}\nonumber\\
&\geq - \eta h^s(x) + K(x,u^\star) \nonumber\\
&\geq - \eta h^s(x) -\Delta h^s(x)\widetilde{W}\phi(x,u)
\end{align}
and hence
\begin{align}\footnotesize\label{app:eq:cbf_learned_approximate_ctrl_constraints3}
u\in \mathcal{U} : 
& \overbrace{L_{\hat{F}}^{\Delta}h^s(x)+L_{\hat{G}}^{\Delta}h^s(x)u +\Delta h^s(x)\widetilde{W}\phi(x,u)}^{h^s\left( \hat{f}(x,u) + \widetilde{W}\phi(x,u) \right)-h^s(x)}
 \nonumber\\
&-  L\bar{\sigma} \sqrt{ 2n \ln\left( \frac{H n}{\delta_s}\right)}
\geq - \eta h^s(x) \nonumber\\
\Rightarrow \qquad & h^s\left( \hat{f}(x,u) + \widetilde{W}\phi(x,u) \right)-  L\bar{\sigma} \sqrt{ 2n \ln\left( \frac{H n}{\delta_s}\right)} \nonumber \\
&\geq (1- \eta) h^s(x)
\end{align}
As Eq.~\ref{app:eq:cbf_learned_approximate_ctrl_constraints3} is equivalent to the constraint in Eq.~\ref{app:eq:safe_policy_class}, we conclude the proof. $\qed$

Note that if the ground truth dynamics $d$ is only state-dependent as assumed in \cite{wang2018safe, berkenkamp2017safe, cheng2019end}, then Eq.~\ref{eq:cbf_learned_approximate_ctrl_constraints} is also linear in control where $L_{x,\phi}=\mathbf{0}$ and feature mapping becomes $\phi(x, u^\star) = \phi(x)$.


\section{Proof of Theorem~\ref{theorem:regret_bound}}\label{sec:app:theorem:regret_bound}

Below we first briefly summarize the theorem of LC$^3$ regret from \cite{kakade2020information} as follows. 
\begin{customthm}{3}\label{theorem:lc3_regret}
(LC$^3$ Regret for finite dimensional, bounded features, See Theorem 1.1 in \cite{kakade2020information}) Consider the finite dimension of $\phi$ as $d_\phi$ and that $\phi$ is uniformly bounded with $\|\phi(x,u)\|_2\leq B$. The LC$^3$ algorithm (Algorithm 1 in \cite{kakade2020information}) enjoys the following expected regret bound:
\begin{equation}\label{eq:lc3_regret}
\begin{split}
    &\mathbb{E}_{\text{LC}^3}[\text{Regret}_T]\\
    \leq& \widetilde{\mathcal{O}}\left(\sqrt{ d_\phi(d_\phi + d_{\mathcal{X}} + H) H^3 T}\cdot \log \left(1+\frac{B^2\|W^\star\|_2^2}{\sigma^2}\right)  \right)
\end{split} %
\end{equation}
where $\widetilde{\mathcal{O}}(\cdot)$ notation drops logarithmic factors in T and H.
\end{customthm}
By revisiting this result, we provide our main statement as follows.
\begin{customthm}{2}\label{app:theorem:regret_bound}[Main Result] Set $\lambda = \bar{\sigma}^2 / C_1^2$. Our algorithm learns a sequence of policies $\pi^0,\dots, \pi^{T-1}$ in $T$ episodes, such that in expectation, we have:
\begin{align*}
\mathbb{E}\left[ \text{Regret}_T \right] \leq \widetilde{\mathcal{O}}\left( H \sqrt{ H r(r + n + H) T}  \right).
\end{align*} Also with probability at least $1- O (\delta_s)$, we have that for all $t\in[T], h^s\in[H]$, $h(x_h^t) \geq -\mathcal{O}(L \epsilon / \eta)$.
\end{customthm}
\begin{proof}
For safety consideration, we proved that the sequence of policies learned from our Algorithm~\ref{alg:optmism_based_learning} satisfying Eq.~\ref{app:eq:cbf_learned_approximate_ctrl_constraints} (and thus Eq.~\ref{app:eq:safe_policy_class}) are all approximately safe, i.e. $h^s(x_h^t) \geq -\mathcal{O}(L \epsilon / \eta)$, with probability at least $1- O (\delta_s)$ for all $t\in[T],h\in[H]$ (See Section~\ref{sec:app:eq:safe_policy_class} and Theorem~\ref{app:prop:approximate_safe}).

For the regret analysis, our proof mainly follows Theorem~\ref{theorem:lc3_regret} for LC$^3$ algorithm and its proofs in \cite{kakade2020information}. Readers are encouraged to refer to \cite{kakade2020information} for more details. One key assumption that allows for regret bound in Eq.~\ref{eq:lc3_regret} lies in the setting of optimism in the face of uncertainty that computes the optimal policy from unconstrained policy class $\Pi$ 
\begin{equation}\label{eq:lc3_double_min}
\pi^t:= \argmin_{ \pi\in\Pi }\min_{{W} \in \texttt{Ball}_t  } J^{\pi}( x^t_0; c, {W} )
\end{equation}
Similarly, in our analysis, by considering the constrained policy class $\Pi_{\widetilde{W}}$ defined in Eq.~\ref{eq:safe_policy_class} and our optimism setup in Eq.~\ref{eq:double_min} analogous to Eq.~\ref{eq:lc3_double_min}, our regret analysis naturally follows LC$^3$ regret in Eq.~\ref{eq:lc3_regret} and enjoys the regret bound with safety guarantee as follows
\begin{align*}
\mathbb{E}\left[ \text{Regret}_T \right] \leq \widetilde{\mathcal{O}}\left( H \sqrt{ H r(r + n + H) T}  \right)
\end{align*}
where $\widetilde{\mathcal{O}}(\cdot)$ notation drops logarithmic factors. Thus we conclude the proof of Theorem~\ref{app:theorem:regret_bound}.
\end{proof}


%





\end{document}